\documentclass[lettersize,journal]{IEEEtran}
\usepackage{amsmath,amsfonts}
\usepackage{amsthm}
\usepackage{algorithmic}
\usepackage{algorithm}
\usepackage{array}
\usepackage[caption=false,font=normalsize,labelfont=sf,textfont=sf]{subfig}
\usepackage{textcomp}
\usepackage{stfloats}
\usepackage{url}
\usepackage{verbatim}
\usepackage{comment}
\usepackage{tablefootnote}
\usepackage{graphicx}
\usepackage{siunitx}
\ifCLASSOPTIONcompsoc
  \usepackage[nocompress]{cite}
\else
  \usepackage{cite}
\fi
\usepackage{amsthm}
\usepackage{booktabs}
\usepackage{enumitem}
\usepackage{makecell}
\usepackage{multirow}
\usepackage{comment}
\usepackage[most]{tcolorbox}
\usepackage[hidelinks]{hyperref}
\newtheorem{theorem}{Theorem}

\newtheorem{proposition}[theorem]{Proposition}
\newcommand{\ord}{\operatorname{rank}}
\newtheorem{definition}{Definition}
\newcommand{\appref}[1]{Appendix~\ref{#1}}

\newtcolorbox{keytakeaway}[1][]{%
  enhanced,
  breakable,
  colback=blue!5!white,
  colframe=blue!75!black,
  boxrule=0.6pt,
  arc=2pt,
  left=4pt,
  right=4pt,
  top=4pt,
  bottom=4pt,
  fonttitle=\bfseries,
  coltitle=black,
  #1
}

\begin{document}

\title{Localising Shortcut Learning in Pixel Space via Ordinal Scoring Correlations for Attribution Representations (OSCAR)}


\author{Akshit~Achara, Peter~Triantafillou, Esther~Puyol-Antón, Alexander Hammers,  Andrew~P.~King, for the Alzheimer’s Disease Neuroimaging Initiative

\IEEEcompsocitemizethanks{
\IEEEcompsocthanksitem Akshit Achara, Esther Puyol-Antón, Alexander Hammers,  Andrew P. King are with the School of Biomedical Engineering and Imaging Sciences, King's College London, WC2R 2LS London, United
Kingdom.\protect\\
E-mail: \{akshit.achara,esther.puyol\_anton, alexander.hammers,\protect\\
andrew.king\}@kcl.ac.uk \protect\\
\IEEEcompsocthanksitem Peter Triantafillou is with the Department of Computer Science, University of Warwick, Coventry, CV4 7AL, United Kingdom. \protect \\
E-mail: p.triantafillou@warwick.ac.uk
}}



\maketitle

\begin{abstract}
Deep neural networks often exploit \emph{shortcuts}. These are spurious cues which are associated with output labels in the training data but are unrelated to task semantics. Shortcuts can inflate in-distribution accuracy but fail under distribution shifts, when the association is absent. When the shortcut features are associated with sensitive attributes, shortcut learning can lead to biased model performance. Existing methods for localising and understanding shortcut learning are mostly based upon qualitative, image-level inspection and assume cues are human-visible, limiting their use in domains such as medical imaging where shortcut features may not be easily perceivable by humans. 

We introduce OSCAR (Ordinal Scoring Correlations for Attribution Representations), a model-agnostic framework for quantifying shortcut learning and localising shortcut features. OSCAR converts image-level task attribution maps into dataset-level rank profiles of image regions and compares them across three models: a balanced baseline model (BA), a test model (TS), and a sensitive attribute predictor (SA). By computing pairwise, partial, and deviation-based correlations on these rank profiles, we
produce
a set of quantitative metrics that characterise the degree of shortcut reliance for TS, together with a ranking of image-level regions that contribute most to it.


Experiments on CelebA natural images (hair colour $\times$ gender), CheXpert chest X-rays (pleural effusion $\times$ sex), and ADNI brain MRI (Alzheimer’s disease $\times$ sex) show that our correlations are (i) stable across seeds and partitions, (ii) sensitive to the level of association between shortcut features and output labels in the training data, and (iii) able to distinguish localised
from diffuse
shortcut features. As an illustration of the utility of our method, we show how worst-group performance disparities can be reduced using a simple test-time attenuation approach based on the identified shortcut regions. OSCAR provides a lightweight, pixel-space audit that yields statistical decision rules and spatial maps, enabling practitioners to test, localise, and mitigate shortcut reliance as part of model evaluation pipelines. The code is available at \url{https://github.com/acharaakshit/oscar}.
\end{abstract}



\section{Introduction}
\label{introduction}
Deep neural networks often exploit \emph{shortcuts}, which are predictive cues spuriously associated with output labels rather than task semantics \cite{geirhos2020shortcut}. Reliance on shortcuts can inflate in-distribution accuracy but fail under distribution shifts, when the association is absent. This phenomenon has been shown in natural image benchmarks (e.g., background-object associations in Waterbirds \cite{sagawa2019distributionally} and attribute associations in CelebA \cite{liu2015deep}) and in diagnostic models in medical imaging where, instead of pathology,  models have been shown to exploit features related to scanner type \cite{zech2018variable,kushol2023effects}, text markers \cite{degrave2021ai}, chest drains~\cite{jimenez2023detecting}, or surgical skin markings~\cite{winkler2019association}.

Attribution methods such as saliency maps \cite{simonyan2013deep} and Grad-CAM \cite{selvaraju2017grad} are frequently used to \emph{visualise} these failures, yet most analyses are qualitative and image-by-image, and they assume that shortcut cues are human-visible. In settings like chest X-ray or brain MRI classification, sensitive attributes (e.g., sex) can be predicted by models while remaining visually opaque to humans, challenging attribution-by-inspection alone.
For example, deep neural networks can classify sex and race from brain MRI~\cite{achara2025invisible, stanley2022fairness} but clinicians are generally not able to identify features to perform these tasks reliably. 
Our goal is to move away from anecdotal inspection of attribution maps to a stable and quantitative automated audit based upon them to identify and understand the nature of shortcut learning, even when its source(s) are complex.


Therefore, in this paper we introduce OSCAR (Ordinal Scoring Correlations for Attribution Representations), a model-agnostic, pixel-space framework that identifies and localises candidate shortcut-related regions. Given a trained model and its training data with associated sensitive attribute(s), we: (i) partition each image into mutually exclusive regions, (ii) summarise the importance of each region in each image by applying a scoring operator such as the mean of attribution values, and (iii) form aggregated rank profiles at the dataset level that we compare across three models: a balanced baseline (BA), a test model (TS), and a sensitive attribute prediction model (SA). We define  pairwise, partial, and deviation-based correlations
based on these profiles, which are used to identify and localise potential shortcut features linked to the sensitive attribute in TS. OSCAR can be combined with any model or attribution method and provides a pixel-space audit of TS that can be used to reduce shortcut learning and improve generalisation.


\subsection{Contributions}
\begin{itemize}[leftmargin=1.2em]
  \item We propose a procedure that turns image-level attribution maps in pixel-space into robust, image-region statistics and rank profiles at the dataset level.
  \item We define a family of aggregated rank profile correlations (pairwise, partial, deviation) with dataset-level inference.
  \item We show how different image partitioning methods (grid, superpixel and atlas) can be embedded into the method and that these can identify shortcut reliance even when shortcut cues are not human-visible.
  \item We demonstrate our method using empirical studies on CelebA natural images (hair colour $\times$ gender), CheXpert chest X-rays (pleural effusion $\times$ sex), and ADNI brain MRI (Alzheimer's disease $\times$ sex), showing that it  quantitatively reflects the degree of shortcut reliance, whether it is diffuse or localised and its origin in pixel space.
  \item We demonstrate the power of the information provided by our method by showing how a simple test-time attenuation technique based on the identified shortcut regions can improve worst group accuracy.
\end{itemize}

This paper significantly extends our previous conference publication~\cite{achara2025invisible}, which introduced the idea of quantitatively analysing demographic shortcut learning using attribution rankings in MRI-based Alzheimer's disease classification. 
Here, we
provide a detailed general formulation of the proposed methodology and present new experiments to validate its effectiveness and illustrate its power across multiple natural and medical image datasets.

\section{Related Work}
\label{sec:related}

\subsection{Shortcut Learning}
Classic demonstrations of shortcut learning include background--object associations in Waterbirds \cite{sagawa2019distributionally} and attribute associations in CelebA (e.g., hair colour with gender) \cite{liu2015deep}. In medical imaging, studies have shown that diagnostic models can exploit shortcuts related to image acquisition details, such as acquisition site in chest X-rays \cite{zech2018variable} or MRI scanner manufacturer \cite{kushol2023effects}. In addition, spurious features such as surgical skin markings in dermatological images \cite{winkler2019association}, or chest drains \cite{oakden2020hidden,jimenez2023detecting} or laterality markings \cite{degrave2021ai} in X-rays have also been shown to induce shortcut learning.

Beyond such obviously visible artefacts, recent work has shown that neural networks can also exploit subtle or latent imaging cues which are not obvious to humans. For example, in  \cite{gichoya2022ai}, the authors demonstrate that race can be predicted from medical images across multiple modalities even when such signals are not perceptible to human experts. This finding implies that sensitive attribute–related features may be embedded broadly across medical imaging datasets, and therefore be available as potential shortcut-inducing features during training. If the sensitive attribute is spuriously associated with the target label in the training data, models may inadvertently rely on these cues, leading to biased performance across sensitive groups.
These phenomena have been documented in chest X-ray classification, mammography, fundus imaging, and neuroimaging, where attributes such as sex, age, or race can act as unintended predictors of clinical labels \cite{degrave2021ai,kushol2023effects,brown2023detecting,achara2025invisible}. In our own prior work \cite{achara2025invisible}, we showed that MRI-based Alzheimer’s disease classifiers exploit demographic shortcuts (sex, race) even when the underlying cues are not visually interpretable, by analysing the alignment between disease and demographic attribution maps. Taken together, these results motivate methods that can detect, characterise, and ideally localise shortcut use in trained models, including settings where the shortcut signal is invisible to humans.

\subsection{Robustness Objectives and Dataset Interventions}
\label{subsection:robustnessobjectives}
A large body of work aims to reduce shortcut reliance by modifying training objectives or use of training data. For example, group distributionally robust optimisation (Group DRO) \cite{sagawa2019distributionally}, representation neutralisation~\cite{sarhan2020fairness}, invariant risk minimization (IRM) \cite{arjovsky2019invariant}, reweighing and rebalancing strategies \cite{idrissi2022simple}, and targeted data augmentation \cite{hendrycks2019augmix} all seek to encourage models to make use of features that are stable across environments, i.e. reduce shortcut learning. While these methods can improve in-distribution
performance, there is growing evidence that such gains do not necessarily persist under genuine distribution shift
\cite{schrouff2022diagnosing}, suggesting that they do not completely eliminate the use of spurious features in learning.
More importantly, these methods do not by themselves provide pixel-space evidence of which cues a trained model actually uses. Our work is complementary to these approaches: it provides a post-hoc, dataset-level test to accept or reject shortcut hypotheses in pixel space, even when cues are human-invisible.

\subsection{Attribution and Pixel-Space Localisation}
\label{subsection:attributionlocalisation}
Attribution methods estimate input importance given a model and prediction. Examples include saliency maps \cite{simonyan2013deep}, guided backpropagation \cite{springenberg2014striving}, layerwise relevance propagation (LRP)~\cite{binder2016layer}, integrated gradients \cite{sundararajan2017axiomatic}, Grad-CAM \cite{selvaraju2017grad} and occlusion-based tests \cite{zeiler2014visualizing}. Prior work on using these methods to understand shortcut learning has proposed qualitative inspections \cite{stanley2022fairness}, and sanity checks for explanation sensitivity \cite{adebayo2018sanity}. However, most analyses are image-by-image and lack statistical hypothesis tests at the dataset level.

Closer to our aims, some studies have aggregated attributions over regions or concepts (e.g., concept activation vectors) to compare models \cite{stanley2022fairness,kim2018interpretability,koh2020concept}. However, these methods require human-defined concepts or segmentations. In~\cite{lapuschkin2019unmasking}, the authors revealed ``Clever Hans'' 
predictors without such prior knowledge by clustering attribution maps to detect recurring spurious strategies. 
Unlike our approach, this method does not provide dataset-level hypothesis tests or the type of rank-based shortcut metrics we propose for localisation.
Our method addresses these limitations by (i) partitioning images into mutually exclusive blocks, (ii) summarising the importance of each block to the model via region-level statistics of attribution values, and (iii) comparing aggregated rank profiles at the dataset level across model triplets using permutation tests.

\subsection{Model Comparisons: Representation and Functional Similarity}
\label{subsection:modelcomparisonrepsim}
Comparisons between models in terms of internal representation  have been proposed as a means of identifying shortcut learning. For example, representation similarity analyses such as canonical correlation analysis or centred kernel alignment compare internal activations across networks \cite{raghu2017svcca,morcos2018insights,kornblith2019similarity}. Functional similarity approaches compare outputs under input perturbations or probes \cite{alain2016understanding,nguyen2016synthesizing}. When these comparisons are made between a `baseline' model and one featuring shortcut learning they can be used to gain insight into the nature of the shortcut features~\cite{boland2024there}. However, these tools reveal where and how networks are similar, but they operate in latent space or function space rather than pixel space. Our approach directly anchors comparisons to image regions and yields spatially interpretable statistics for testing shortcut alignment between models.

\subsection{Quantifying Bias and Spurious Feature Reliance}
\label{subsection:quantifyingbias}
A parallel literature quantifies dataset bias and spurious associations \cite{torralba2011unbiased}, discovers spurious attributes \cite{sohoni2020no}, or regularises gradients to be ``right for the right reasons'' \cite{ross2017right}. Compared to these, we do not require annotations for spurious regions, discovered environments, or gradient penalties. Instead, we use the idea that, if a test model uses a shortcut, then the \emph{ordering} of regional attribution importance will align with that of a sensitive attribute classifier. We formalise this via pairwise, partial, and deviation-based correlations over
dataset-level region statistics of attributions.

\subsection{Shortcut Testing and Bias Encoding}
\label{subsection:shortcuttesting}
Beyond documenting shortcuts, recent work has proposed tests to distinguish the encoding of sensitive attributes from actual shortcut reliance. For example, in ~\cite{brown2023detecting}, the authors propose a technique for shortcut testing, showing in clinical tasks that models often encode attributes like age or sex, but that this does not always drive unfair predictions. Complementarily, in~\cite{stanley2025and}, the authors use synthetic neuroimaging data to examine where and how biases are encoded across layers, finding that not all encoded signals are used for the main task. Beyond these works, several post-hoc methods explicitly aim to uncover shortcut behaviour, such as medical imaging studies that 
identify shortcut learning by inspecting the network layer at which a task is learnt in different models~\cite{boland2024there}.

Our framework is aligned with these efforts but differs in being fully post hoc, pixel-space, and model-agnostic: it tests for shortcut-related alignment directly at the dataset level, providing correlational evidence about the extent to which identified cues may be relied upon.

\subsection{Design Desiderata for Shortcut Auditing}
\label{subsection:designdesiderata}
In summary, prior work has demonstrated (i) training for robustness, (ii) comparing internal representations, or (iii) visualising attributions, often qualitatively. However, existing methods have shortcomings when considered in the context of automated model auditing. To highlight this, we establish the following key design desiderata that a practical shortcut auditing method should satisfy:
\begin{enumerate}[label=D\arabic*, leftmargin=1.5em]
    \item \textbf{No shortcut annotations.} The method should not require predefined concept activation vectors, segmentation masks, or manually annotated shortcut regions.
    \item \textbf{Dataset-level analysis.} The method should operate at the dataset level, providing population-level evidence about shortcut reliance rather than relying only on image-by-image inspection.
    \item \textbf{Pixel-space localisation.} The method should localise shortcut-related evidence in pixel space, so that implicated regions can be visualised and related to domain knowledge.
    \item \textbf{Quantitative, model-agnostic metrics.} The method should provide quantitative, statistically interpretable metrics and be applicable to arbitrary architectures and attribution methods.
\end{enumerate}

Viewed through these desiderata, limitations of existing methods become clear. Robustness and data-centric approaches (Section~\ref{subsection:robustnessobjectives}) mitigate shortcuts but lack pixel-space, post-hoc evidence (D3, D4). Attribution-based analyses (Section~\ref{subsection:attributionlocalisation}) support localisation (D3) but rarely provide dataset-level inference or statistical testing (D2, D4). Representation, bias-quantification, and shortcut-testing methods (Sections~\ref{subsection:modelcomparisonrepsim}–~\ref{subsection:shortcuttesting}) yield quantitative diagnostics but operate in latent or feature space and do not deliver a unified, pixel-space, dataset-level audit of shortcut reliance (D3, D4).

OSCAR is a pixel-space, model-agnostic procedure that aggregates attribution maps over disjoint regions, orders aggregated regions, and performs  rank-based correlation tests across a (BA, TS, SA) model triplet. This yields dataset-level decision rules (with $p$-values) for accepting or rejecting shortcut hypotheses and localises the implicated regions, addressing all of desiderata D1-D4.


\section{Methods}

\subsection{Notation \& Setup}
\label{sectLnotationsetup}
The input to our method is a set of training data $\mathcal D = (\mathcal{X}, \mathcal{Y}, \mathcal{A})$, comprising images $\mathcal X$, output labels $\mathcal Y$ and sensitive attribute values $\mathcal A$.

The method relies on three trained models:
\begin{enumerate}
    \item \textbf{Baseline Model (BA)}: A model to classify the target label $\mathcal Y$ from the images $\mathcal X$, trained using a training set from $\mathcal D$ which is  balanced by the sensitive attribute $\mathcal A$.
    \item \textbf{Test Model (TS)}: The model being audited. This is trained to classify the target labels $\mathcal Y$ from the images $\mathcal X$ using training data from $\mathcal D$ that is potentially biased.
    \item \textbf{Sensitive Attribute Model (SA)}: A model trained to classify the sensitive attribute $\mathcal A$ from the images $\mathcal X$. 
\end{enumerate}
The method aims to investigate the hypothesis that TS uses shortcut learning features linked to the sensitive attribute $\mathcal A$.

\subsection{Overview}
\label{subsection:methodoverview}
The overall OSCAR framework is depicted in Figure~\ref{fig:framework}. There are six main steps: (i) training BA, TS and SA models, (ii) estimating attribution maps for each model for each test set sample, (iii) partitioning the maps into disjoint regions to form ranked sets, (iv) aggregating the ranked sets into dataset level rank profiles for each of BA, TS and SA, (v) computing correlations and testing hypotheses using these rank profiles, and (vi) producing region contribution score (RCS) maps to illustrate the shortcut regions. Each of these stages is described in more detail in the subsections below.

\begin{figure*}
\centering
\resizebox{.85\textwidth}{!}{
  \includegraphics[width=\linewidth]{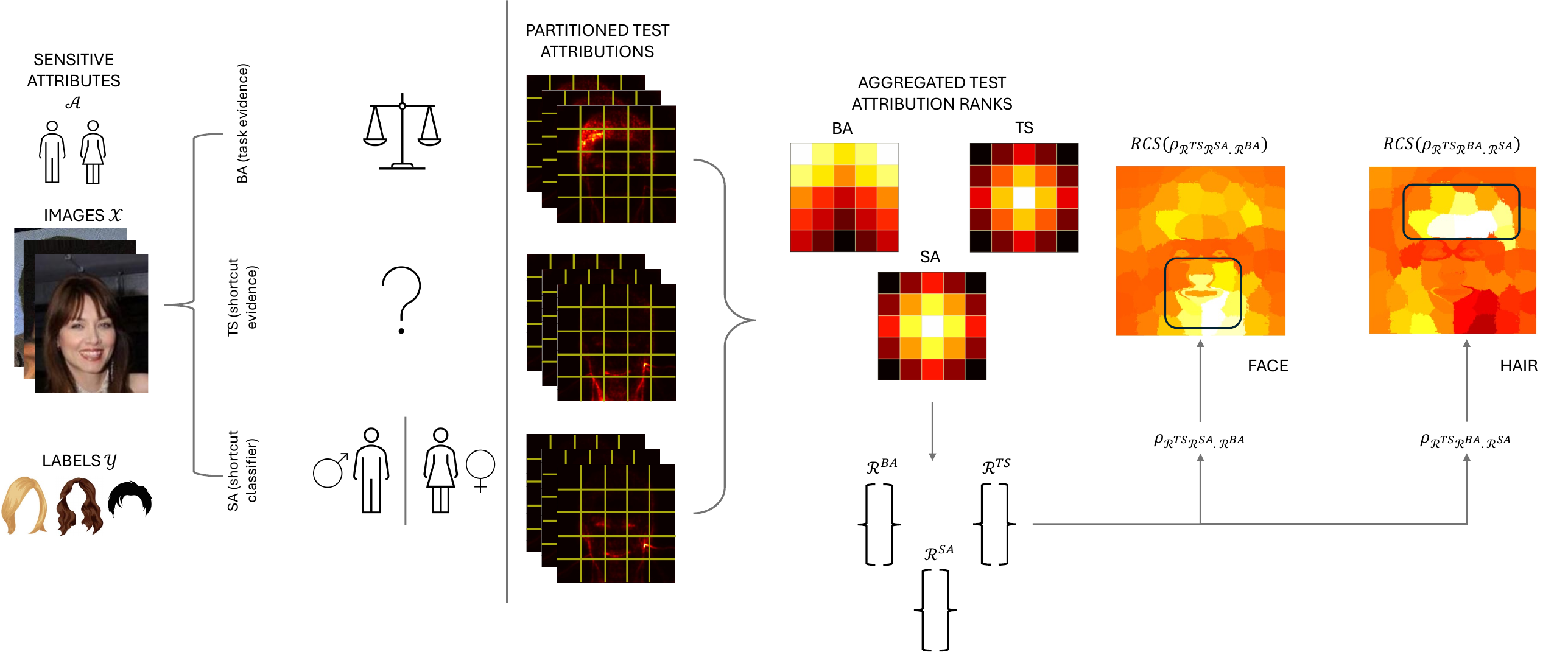}}
  \caption{Overview of the OSCAR framework.  (i) From the input data, three different models are trained: TS, SA and BA (Section ~\ref{subsection:methodstraining}). (ii) These models are used to obtain attribution maps for each test set sample (Section~\ref{subsection:interpretabmethod}). (iii) The attribution maps are partitioned into disjoint regions and the attribution values summarised at the region level to form a ranked set for each image (Section~\ref{subsection:partitions}). (iv) The ranked sets are aggregated across the test set to form dataset level rank profiles $\mathcal R$ for each of TS, SA and BA (Section~\ref{subsection:rankanalysis}). (v) Correlations are computed using the rank profiles and shortcut hypotheses tested (Section~\ref{subsection:hypotheses}). (vi) A corresponding region contribution score (RCS) map is obtained, which indicates the region(s) that contributed to the shortcut learning (Section~\ref{subsection:rcs}).}
  \label{fig:framework}
\end{figure*}

\subsection{BA, TS and SA Model Training}
\label{subsection:methodstraining}

TS is the model being audited and we assume it has already been trained to predict $\mathcal Y$ using a training set comprised of data from $\mathcal D$ with an arbitrary distribution of $\mathcal A$ values. This model may exhibit shortcut learning and is potentially biased by $\mathcal A$, in the sense that the joint distribution of $(\mathcal Y,\mathcal A)$ in its training data can be highly imbalanced (see Table~\ref{tab:pergroup-train} for examples from our experiments), so one $(\mathcal Y,\mathcal A)$ group dominates the others.
To train BA, we curate a training set which is approximately balanced by the sensitive attribute $\mathcal A$, i.e. the empirical label distributions are similar across attribute groups, $\hat P(\mathcal Y \mid \mathcal A = A_0) \approx \hat P(\mathcal Y \mid \mathcal A = A_1)$\footnote{$\hat P(\cdot)$ represents an empirical label distribution.}.
Finally, the training set for SA is the same as that used to train BA but the output labels used are $\mathcal A$ instead of $\mathcal Y$.

\subsection{Estimating Attribution Maps}
\label{subsection:interpretabmethod}
The next stage is to estimate attribution maps for each of BA, TS and SA across the hold-out independent test set from $\mathcal D$.
These are estimated using an interpretability method that is applicable to the neural networks, with inputs $\mathcal X$ and the predicted labels $\mathcal{\hat{Y}}$.
While our framework is agnostic to the choice of the interpretability method, it does require one. The interpretability method produces attributions in pixel space, which we denote by $\mathcal L$.
These are produced for each model for each image in the hold-out test set.

\subsection{Attribution Map Partitioning}
\label{subsection:partitions}

We require that the image $x \in \mathbb{R}^{H \times W}$ (or $\mathbb{R}^{H \times W \times D}$) is partitioned into $n$ mutually exclusive blocks $\mathcal B_r$ where $|\mathcal B_r| > 0$ and $x = \bigsqcup_{r=1}^n \mathcal B_r$. Each block $\mathcal B_r$ in $\mathcal L$ is summarised by a single attribution score $s_r = S(\mathcal L \cap \mathcal B_r)$ using an operator $S$. For example, in our experiments we use the mean operator for $S$. This leads to an image level ordered set $s = \ord(\{s_1,\ldots, s_n\})$ where $\ord$ returns the per-region ranks (one rank per block, $1$ = most important, $n$ = least important).
The rank vectors are computed for each image in the test set.


We propose and investigate three different partitioning techniques to obtain the blocks $\mathcal B_r$: grid-based  partitioning, superpixel partitioning and atlas-based partitioning, as described below.

\begin{itemize}

\item \emph{Grid-based partitioning}:
In this case, we partition each image into a regular grid of fixed-sized blocks $\mathcal B_r$ where $|\mathcal B_r| = |\mathcal B_{r^{\prime}}|\ \forall r,r^{\prime}\in \{1,\ldots,n\}$.

\item \emph{Superpixel partitioning}:
In this case, we partition each image into mutually exclusive blocks $\mathcal{B}_r$ using a superpixel-based approach. We construct a reference edge image $\tilde I$ by averaging Sobel edge maps over the test sets, and apply SLIC (Simple Linear Iterative Clustering)~\cite{achanta2012slic} to $\tilde I$ to obtain a content-aware partition of the pixel grid into $n$ blocks. Formally,
\[
\tilde I = \mathrm{avg\_sobel}\!\left(\{I^{(i)}\}\right), \qquad
\{\mathcal{B}_r\} = \mathrm{SLIC}\!\left(\tilde I;\, n\right),
\]
Here, $i\in\{1,\ldots,m\}$ is the image index.

\item \emph{Atlas-based partitioning}:
In this case, we partition each image using a predefined semantic or anatomical atlas. Each block $\mathcal B_r$ corresponds to all pixels assigned to atlas region $r$, and the blocks are mutually exclusive but not necessarily equal in size. This option provides region definitions that are anatomically meaningful and consistent across images.

\end{itemize}

\subsection{Forming Aggregated Rank Profiles}
\label{subsection:rankanalysis}


We now combine (aggregate) the image level ordered sets $s$ for each image in the hold-out test set to form a single dataset level rank profile for each model. This profile summarises the regions generally used by the model for inference.

Here, we assume ranked sets $s_i^M$ are obtained for each image $x_i$ and model $M \in \{BA,TS,SA\}$. That is, for each image we have three ranked sets $s_i^{TS}, s_i^{BA}\ \text{and}\ s_i^{SA}$, which contain the ordered block indices based on attribution importance. These ranked sets are aggregated over all images in the hold-out test set by an operator such as the median, i.e. $\mathcal R^M = \operatorname{median} (s_i^M)$. This results in three dataset-level ranked sets of regions, which we denote by $\mathcal R^{TS} , \mathcal R^{BA}\ \text{and}\ \mathcal R^{SA}$.

\subsection{Correlations and Hypothesis Testing}
\label{subsection:hypotheses}

The ranked sets $\mathcal R^{TS} , \mathcal R^{BA}\ \text{and}\ \mathcal R^{SA}$ are used to compute interpretable statistics, specifically pairwise correlation, partial correlation and deviation-based correlation.




Below we present a general formulation of the correlations we propose, using $(A,B,C)$ to represent the three models.
In practice, one typically sets
$(A,B,C)=(\text{TS},\text{SA},\text{BA})$, i.e., $A$ is the potentially biased test model, $B$ predicts the sensitive attribute, and $C$ is the baseline model. The statements below, however, are written generically and hold for any $(A,B,C)$.

Let $\mathcal{R}^{M}\in\mathbb{R}^{n}$ denote the
test-set rank profile over the $n$ regions for model $M\in\{A,B,C\}$ as constructed in Section~\ref{subsection:rankanalysis}. Let $\rho(\cdot,\cdot)$ denote a correlation between two such region-wise rankings.
Intuitively, a larger $\rho$ means the two models order regions more similarly.

\begin{definition}[Partial correlation on aggregated ranks]
\label{def1:pcorr}
Let $\mathbf e_A$ and $\mathbf e_B$ be the residual vectors after
regressing $\mathcal R^A$ and $\mathcal R^B$ on $\mathcal R^C$ (with an intercept).
We define the partial correlation as the Pearson's correlation\footnote{Pearson's correlation is used for the main experiments. Justification and alternate choices are discussed in Appendix \ref{appendix:salscoreagg}.} between
$\mathbf e_A$ and $\mathbf e_B$:
\[
\rho_{AB\cdot C}
=
\frac{\sum_{r=1}^n (e_{A,r} - \bar e_A)(e_{B,r} - \bar e_B)}
{\sqrt{\sum_{r=1}^n (e_{A,r} - \bar e_A)^2}\;
 \sqrt{\sum_{r=1}^n (e_{B,r} - \bar e_B)^2}}.
\]
\end{definition}

Based on the above definition, we now define three different measures of correlation on the rank profiles of the models $(A,B,C)$.

\begin{enumerate}[label=\arabic*., leftmargin=1.2em, labelsep=0.5em]

  \item \textbf{Pairwise correlation - agreement: $\boldsymbol{\rho(\mathcal{R}^{A},\,\mathcal{R}^{B})}$.}
  \begin{enumerate}[label=\alph*)]
    \item \textbf{Null hypothesis.} The region-wise rankings of $A$ and $B$ are not
    associated (``$A$ and $B$ do not rank regions in a similar or opposite manner'').
    \item \textbf{Alternate hypothesis.} $A$ and $B$ rank regions in a similar or opposite manner.
  \end{enumerate}
  \textit{Interpretation.} Rejecting the null hypothesis indicates that $A$ and $B$ share a common
  region-wise ordering of importance. With $(A,B)=(\text{TS},\text{SA})$ this supports the hypothesis that
  the test model uses similar evidence to the sensitive attribute model.

  \medskip

  \item \textbf{Partial correlation - shared structure beyond a reference: $\boldsymbol{\rho_{\mathcal{R}^{A}\mathcal{R}^{B}.\,\mathcal{R}^{C}}}$}
  \begin{enumerate}[label=\alph*)]
    \item \textbf{Null hypothesis.} Any apparent similarity between $A$ and $B$ is explained by
    their common similarity to $C$ (``no additional shared structure beyond $C$'').
    \item \textbf{Alternate hypothesis.} $A$ and $B$ share region-wise structure that is \emph{not}
    explained by $C$.
  \end{enumerate}
  \textit{Interpretation.} This asks whether $A$ and $B$ agree in the \emph{residual}
  ordering once the structure captured by $C$ is factored out. With
  $(A,B,C)=(\text{TS},\text{SA},\text{BA})$, rejecting the null hypothesis indicates that the
 TS and SA model attribution alignment is not merely inherited from BA.

  \medskip
\item \textbf{Deviation-based correlation - $\boldsymbol{\rho_{\text{dev}}}$}.
Here, we are interested in whether the way model $A$ departs from a reference model $C$
is aligned with the region-wise ordering of a third model $B$.
Formally, we compute the correlation where
we first regress $\mathcal R^{A}$ on $\mathcal R^{C}$ (with an intercept)
and take the residuals
$
e_{A} = \mathcal R^{A} - \widehat{\mathcal R^{A}}(\mathcal R^{C}),
$
then define
\[
\rho_{\text{dev}} = \mathrm{Corr}\bigl(e_{A},\, \mathcal R^{B}\bigr).
\]
\begin{enumerate}[label=\alph*)]
  \item \textbf{Null hypothesis.} The residual deviations of $A$ from $C$ are not associated
  with the rankings of $B$.
  \item \textbf{Alternative hypothesis.} The residual deviations of $A$ from $C$ are
  aligned with the rankings of $B$.
\end{enumerate}
\textit{Interpretation.} With $(A,B,C)=(\text{TS},\text{SA},\text{BA})$, a large positive
$\rho_{\text{dev}}$ indicates that the part of TS's rank profile that cannot be explained
by BA follows the region-wise ordering of SA. We refer to this as a deviation-based correlation since it quantifies how TS deviates from BA in the direction of SA. Compared to the partial correlation, only $A$ is residualised on $C$, while $B$ is left in its original scale.

\end{enumerate}

\noindent\textbf{Assumptions}\\ 
In the above formulation we make the following assumptions:
\begin{enumerate}
    \item[*] Images are approximately spatially aligned so regions are comparable across samples. 
    \item[*] Balancing the training data for BA with respect to $\mathcal A$ reduces the incentive to exploit shortcuts linked to $\mathcal A$. Note that we do not assume that BA is completely free of shortcut behaviour or sensitive attribute encoding. Throughout this work, we interpret BA as an approximate ``reference'' model: partial and deviation-based correlations quantify how TS aligns with SA beyond what can be explained by this reference, rather than as a formal proof that all remaining alignment must be shortcut-related.

\end{enumerate}

\noindent\textbf{Hypothesis Testing}\\
We assess significance using two complementary procedures.  
First, we compute two-sided permutation $p$-values based on 10{,}000 region-index permutations of the aggregated rank profiles.  
Second, we quantify dataset-level variability via image-level bootstrap resampling
(10{,}000 resamples), from which we report 95\% confidence intervals.

\subsection{Region Contribution Scores} 
\label{subsection:rcs}
Finally, having computed the correlations and tested shortcut hypotheses, we now produce interpretable maps which indicate the nature and locality of the shortcut features used by TS. We refer to these maps as Region Contribution Score (RCS) maps.


To produce the RCS maps, we compute an RCS value per region $r$ as the signed contribution to the partial correlation numerator.  High magnitude regions indicate dominant contributors to the measured correlation, and hence to the detected shortcut related alignment (positive magnitude).

Formally, let $\mathbf e_A$ and $\mathbf e_B$ be the residual vectors after regressing
$\mathcal R^{TS}$ and $\mathcal R^{SA}$ on $\mathcal R^{BA}$.
We form z-scored residuals
\[
   z_{A,r} = \frac{e_{A,r} - \bar e_A}{\sigma_A}, \qquad
   z_{B,r} = \frac{e_{B,r} - \bar e_B}{\sigma_B},
\]
and define
\[
   \mathrm{RCS}(r) = z_{A,r} z_{B,r}.
\]
Since
\(
   \rho_{TS,SA.BA} = \frac{1}{n-1} \sum_r z_{A,r} z_{B,r},
\)
each $\mathrm{RCS}(r)$ is proportional to that region’s contribution to the partial correlation. Additionally, we apply an $\ell_1$-normalisation so that $\sum_r |\mathrm{RCS}(r)| = 1$. This preserves the sign of each region’s effect while placing magnitudes on a common scale across regions.

The overall rank aggregation and correlation computation methods are also summarised in Algorithm~\ref{alg:audit}. 

\begin{algorithm}[t]
\caption{Aggregated Rank Partial Correlation}
\label{alg:audit}
\begin{algorithmic}[1]
\STATE Input: models $M\in\{\text{BA},\text{TS},\text{SA}\}$, attribution method $\mathcal{L}$, partition $\{\mathcal B_r\}_{r=1}^n$
\FOR{image $i=1,\dots,m$}
  \FOR{each model $M \in \{\text{BA},\text{TS},\text{SA}\}$}
    \STATE Compute attribution map $\mathcal L_i^M$ for $x_i$
    \FOR{region $r=1,\dots,n$}
      \STATE $s_{i,r}^M \leftarrow S(\mathcal L_i^M \cap \mathcal B_r)$ \quad \COMMENT{e.g., mean}
    \ENDFOR
    \STATE $\mathbf s_i^M \leftarrow \ord(s_{i,1}^M,\dots,s_{i,n}^M)$
  \ENDFOR
\ENDFOR
\FOR{each model $M \in \{\text{BA},\text{TS},\text{SA}\}$}
  \STATE $\mathbf R^M \leftarrow \operatorname{median}_i (\mathbf s_i^M)$
\ENDFOR
\STATE $\rho_{\text{TS},\text{SA}.\text{BA}} \leftarrow \mathrm{pcorr}(\mathbf R^{\text{TS}},\mathbf R^{\text{SA}} \mid \mathbf R^{\text{BA}})$
\STATE RCS$(r) \leftarrow$ signed contribution via residual product
\STATE \textbf{return} $\rho$, permutation $p$, confidence interval (CI) and RCS map
\end{algorithmic}
\end{algorithm}

Theoretical properties are discussed in~\appref{appendix:theory}.

\section{Experimental Details}
\label{section:experiments}

\subsection{Datasets \& Preprocessing}

We evaluate OSCAR on two-dimensional (2D) natural images, 2D medical images (chest X-rays) and three-dimensional (3D) medical images (brain MRI). In every setting, the target is a binary label $\mathcal{Y}\in\{0,1\}$ and the sensitive attribute is a binary variable $\mathcal{A}\in\{0,1\}$. All train/validation/test splits in the medical imaging datasets are patient/subject-wise disjoint. The specific datasets, tasks and sensitive attributes used in our experiments are described below and the data distributions are detailed in Tables~\ref{tab:datasets} and~\ref{tab:pergroup-train}.


\subsubsection{CelebA}
\label{sec:celeba}

The CelebA dataset~\cite{liu2015deep} consists of natural 2D images of faces and we study the binary target label of \emph{Hair} (0 = non-blond, 1 = blond) and the sensitive attribute of \emph{Gender} (0 = female, 1 = male). We centre-crop all images to the face region, resize to $224{\times}224$, and standardise images based on the ImageNet statistics.
We introduce an association between female gender and blond hair in the TS model, whilst BA and SA are balanced with respect to hair/gender composition, as shown in Tables~\ref{tab:datasets} and~\ref{tab:pergroup-train}.

\subsubsection{CheXpert}
\label{sec:chexpert}

The CheXpert dataset~\cite{irvin2019chexpert} consists of 2D chest X-ray images. We study the clinical target label of \emph{Pleural Effusion} (0 = absent, 1 = present) and the sensitive attribute of \emph{Sex} (0 = female, 1 = male). We restrict our dataset to AP (anterior-posterior) frontal views and exclude PA (posterior-anterior) and lateral images resized to $224\times224$.
Stratification is performed by age, i.e. subjects are proportional based on age (subjects $< 45$ are categorised as young and others as old)  across the train, validation and test splits. In the CheXpert experiments, the association introduced into TS is between pleural effusion and sex, whilst the BA and SA are balanced with respect to pleural effusion/sex composition.

\subsubsection{ADNI}
\label{sec:adni}

The ADNI dataset~\cite{petersen2010alzheimer} consists of 3D brain MRI images acquired from healthy controls, and subjects with mild cognitive impairment and Alzheimer's disease. We used the clinical target label of \emph{Alzheimer's Disease} (0 = cognitively normal, 1 = Alzheimer's disease present) and the sensitive attribute of \emph{Sex} (0 = female, 1 = male).
We resize the T1-weighted skull-stripped (using~\cite{isensee2019automated,isensee2021nnu}) volumetric inputs to size $256\times 256\times 256$ and perform z-score standardisation. The data splits were stratified based on age and race, i.e. the number of subjects in the train, validation and test sets were proportional based on race and age. In the ADNI experiments the association introduced into the TS model was between Alzheimer's disease and sex, whilst the BA and SA models are trained on subject sets that are significantly more balanced (compared to TS) with regard to Alzheimer’s disease × sex at the subject level. Due to the limited number of available subjects in the minority groups, it was not possible to construct both a heavily skewed TS split and a perfectly balanced BA split with exactly the same number of training subjects. BA is therefore constructed to be as balanced as possible under these constraints, which leads to a different overall training set size than TS.

\begin{table*}[!t]
  \centering
  \caption{Summary of datasets used in the study. For CelebA and CheXpert, the train/validation/test entries give the number of images. For ADNI, the train and validation columns show the number of scans used for the baseline (BA) and test (TS) models in the format \textit{BA $\lvert$ TS}; the test column reports the number of subjects in the shared test set. PF=Pleural effusion. CN=Cognitively normal. AD=Alzheimer's disease.}
  \label{tab:datasets}
  \setlength{\tabcolsep}{6pt}    
  \renewcommand{\arraystretch}{1.15} 
  \begin{tabular}{l l l r r r}
    \toprule
    \textbf{Dataset} &
    \textbf{$\mathcal Y$ (0/1)} &
    \textbf{$\mathcal A$ (0/1)} &
    \textbf{train} &
    \textbf{val} &
    \textbf{test} \\
    \midrule
    CelebA -- Hair $\times$ Gender &
    \makecell[l]{0 = non-blond\\1 = blond} &
    \makecell[l]{0 = female\\1 = male} &
    10{,}000 & 500 & 1{,}000 \\
    CheXpert -- PF $\times$ Sex &
    \makecell[l]{0 = no PF\\1 = PF} &
    \makecell[l]{0 = female\\1 = male} &
    10{,}000 & 500 & 1{,}000 \\
    ADNI -- Alzheimer's Disease $\times$ Sex &
    \makecell[l]{0 = CN\\1 = AD} &
    \makecell[l]{0 = female\\1 = male} &
    1{,}289 $\lvert$ 1{,}347 & 149 $\lvert$ 161 & 502 \\
    \bottomrule
  \end{tabular}
\end{table*}

\begin{table}[!ht]
  \centering
  \caption{Per-group training counts for the baseline (BA) and test (TS) models. Columns list counts for $(y,a)\in\{(0,0),(0,1),(1,0),(1,1)\}$, where $y$ is the target label and $a$ the sensitive attribute. For CelebA and CheXpert, entries are image counts. For ADNI, entries are given as \textit{scans(subjects)}, i.e. the number of scans followed in parentheses by the corresponding number of MRI subjects in that group.}

  \label{tab:pergroup-train}
  \small
  \setlength{\tabcolsep}{4pt}
  \renewcommand{\arraystretch}{1.12}
  \begin{tabular}{@{} l @{\hspace{1.2em}} l r r r r r}
    \toprule
    \textbf{Dataset} & \textbf{Train set} & \textbf{(0,0)} & \textbf{(0,1)} & \textbf{(1,0)} & \textbf{(1,1)} \\
    \midrule
    \multirow{2}{*}{CelebA}
      & BA & 2{,}500 & 2{,}500 & 2{,}500 & 1{,}374\tablefootnote{This is the maximum number of Male+Blond samples available for the training set.} \\
      & TS   & 25 & 4{,}975 & 4{,}975 & 25 \\
    \addlinespace
    \multirow{2}{*}{CheXpert}
      & BA & 2{,}500 & 2{,}500 & 2{,}500 & 2{,}500 \\
      & TS   & 25 & 4{,}975 & 4{,}975 & 25 \\
    \addlinespace
    \multirow{2}{*}{ADNI}
      & BA & 433(86) & 449(88) & 182(57) & 225(72) \\
      & TS   & 859(174) & 21(4) & 8(2) & 459(140) \\
    \bottomrule
  \end{tabular}
\end{table}

\subsection{Model Training}


To train the BA, TS and SA models (see Section~\ref{subsection:methodstraining}), the following neural network architectures were used: ResNet50~\cite{he2016deep} and InceptionV3~\cite{szegedy2016rethinking} for the 2D experiments, and 3D-ResNet50 for the 3D experiment. We did not evaluate InceptionV3 in the 3D experiment as it has no standard 3D implementation. We also provide some additional results using VGG16~\cite{simonyan2014very} and MobileNetV3-Large~\cite{howard2019searching} in ~\appref{appendix:interpretabilitymethods}. In addition, we evaluate a Vision Transformer~\cite{dosovitskiy2020image} (ViT) with Attention-aware LRP~\cite{pmlr-v235-achtibat24a}, and show that OSCAR extends naturally to transformer architectures. See \appref{appendix:transformermodels} for full details and results.

All models were trained using the Adam optimiser with a learning rate of $5e-5$. The batch size was $32$ for the 2D models and $4$ with gradient accumulation for the 3D models. Early stopping based on the validation F1-score was used, with a patience of $10$ for the 2D experiments and $50$ for the 3D experiments. We initialised the networks with pretrained ImageNet weights and used a dropout rate of $0.2$ for all networks.

Further training, attribution layer targets, and optimisation settings are detailed in~\appref{appendix:computationaldetails}.

\subsection{Evaluation}
\label{subsection:evaluation}

Unless stated, all evaluation metrics are reported on a test set that is balanced in $\mathcal{A}$ within each class of $\mathcal{Y}$. Exact counts and skew levels are provided per dataset in Tables~\ref{tab:datasets} and~\ref{tab:pergroup-train}.

To compare the BA and TS models
we report worst group accuracy,
defined as the lowest accuracy over the four
$(\mathcal Y,\mathcal A)$ subgroups (i.e., the target–attribute pairs
$(0,0)$, $(0,1)$, $(1,0)$, $(1,1)$), as this is related to the robustness of the models to spurious correlations. This serves as a simple performance-based
baseline for shortcut behaviour: if the TS model relies on the
$\mathcal{Y}$–$\mathcal{A}$ shortcut, its worst-group accuracy on the balanced test
set tends to be substantially lower than that of the balanced BA model.

It is common to quantify how strongly sensitive attributes are encoded in models by training simple classifiers on intermediate representations~\cite{song2019overlearning,stanley2025and}. 
Following this practice, we train a logistic regression probe on the TS penultimate-layer features to predict $\mathcal{A}$ on the balanced test set and report leakage results in Section~\ref{subsection:shortcutdetection}.

\subsection{Attribution Maps}
\label{attributionmaps}

We used a Grad-CAM \cite{selvaraju2017grad} layer positioned at the final convolution block as the default attribution technique. For the 3D experiments, we used a volumetric Grad-CAM variant. Attribution maps were ReLU-thresholded and upsampled to image/volume resolution. We $\ell_1$-normalised all maps per image before further processing. We additionally provide results using LRP in~\appref{appendix:interpretabilitymethods}.

\subsection{Partitions} 
\label{partitions_refined}


For the grid-based partitioning, we used block sizes of $8\times8$ and $16\times16$. This partitioning method was evaluated on all 2D experiments.

For the superpixel partitioning. we used $k\in\{64,256\}$ for the experiments. Here, $k$ represents the number of regions in the partition.
The superpixel partitioning was also evaluated on all 2D experiments.

The atlas-based partitioning can be utilised whenever there is a pre-defined atlas available for the dataset. In our experiments, this was the case for the 3D brain MRI data. We employed the Hammersmith Atlas Database~\cite{hammers2003three}, in its 95 region version\footnote{\url{https://brain-development.org/brain-atlases/adult-brain-atlases/}}. All test images were registered to the atlas and the resulting transformations were applied to the attribution maps.

\subsection{Region Statistic Rank Profiles} 
\label{rankprofiles}

Within each region, we summarize attributions by a region-level statistic $S$ (see Section~\ref{subsection:partitions}). In our experiments, we used the mean attribution per region for $S$. However, alternatives such as saliency score~\cite{stanley2022fairness} can also be used. We evaluate a saliency score aggregation variant in~\appref{appendix:salscoreagg}.

\section{Results}

\subsection{Shortcut Detection using Classification Performance}
\label{subsection:shortcutdetection}
\begin{figure}[t]
\centering
\resizebox{0.4\textwidth}{!}{
  \includegraphics[width=\linewidth]{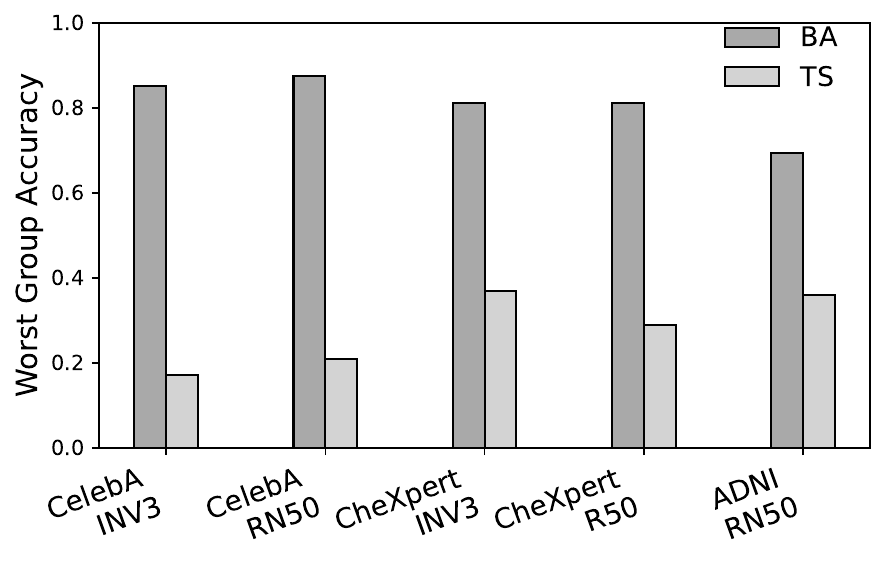}}
  \caption{Worst group accuracy for classification of target labels. RN50 stands for ResNet50 and INV3 for InceptionV3.}
  \label{fig:WGA}
\end{figure}
First, we report the quantitative accuracy of the TS and BA models on the balanced test sets for all datasets/tasks to ascertain whether shortcut learning is taking place. Figure~\ref{fig:WGA} presents the worst-group accuracies, i.e., the minimum accuracy
over the four $(\mathcal Y,\mathcal A)$ groups
 of the BA and TS models for all three datasets/tasks. The differences in worst group accuracy indicate bias against minority groups that is consistent with shortcut learning in the TS models. The highest disparities are seen for the CelebA and CheXpert datasets. The sensitive attribute classification performance using the SA models is reported in~\appref{appendix:sensitiveattributeclass}.

 We report the AUCs for sensitive attribute leakage from the TS model features in Table~\ref{tab:leakage_auc}. The probe
AUC quantifies how strongly $\mathcal{A}$ is encoded
in the representation.  We observe consistently high leakage that
decreases systematically (with negligible impact for CelebA)  as the number of discordant pairs increases, indicating
that increasing the $(\mathcal Y,\mathcal A)$ balance in the training set reduces but does not eliminate the linear decodability
of $\mathcal{A}$. We also compute leakage AUC for the balanced baseline (BA). Interestingly, BA exhibits similar attribute-leakage AUC to the TS models, despite having substantially higher worst group accuracy and much lower OSCAR partial correlation (see Section~\ref{varyinganticorr}). This supports that representation-level leakage is not a reliable indicator of shortcut reliance~\cite{brown2023detecting}, and also supports our claim that OSCAR’s correlations and group robustness metrics capture a fundamentally different phenomenon.

\begin{table}[t]
\centering
\caption{Linear probe AUC for predicting the sensitive attribute $A$ from penultimate-layer features of the biased task model (TS) on the balanced test set. We vary the number of discordant $(Y,A)$ pairs in the training data. Higher AUC indicates stronger attribute leakage.}
\label{tab:leakage_auc}
\resizebox{0.45\textwidth}{!}{
\begin{tabular}{lcccccc}
\toprule
\multirow{2}{*}{Dataset / Backbone} & \multicolumn{6}{c}{\# discordant $(Y,A)$ pairs in training} \\
\cmidrule(lr){2-7}
& 25 & 500 & 1000 & 1500 & 2000 & 2500 (BA)\\
\midrule
CelebA / ResNet50        & 0.982 & 0.985 & 0.978 & 0.977 & 0.975 & 0.972\\
CheXpert / ResNet50      & 0.924 & 0.917 & 0.901 & 0.899 & 0.908 & 0.909\\
CelebA / InceptionV3     & 0.944 & 0.955 & 0.952 & 0.945 & 0.959 &0.942\\
CheXpert / InceptionV3   & 0.877 & 0.817 & 0.792 & 0.780 & 0.773 & 0.758\\
ADNI / ResNet50   & 0.793 & 0.778 & 0.801 & 0.739 & 0.717 & 0.756 \\
\bottomrule
\end{tabular}}
\end{table}

We also observe that removing the sensitive attribute direction from the TS classifiers does not affect the classification performance (see~\appref{appendix:sensitiveattributeclass}).  

\begin{keytakeaway}[]
\textbf{Conclusion}: Group-level performance disparities are present in our experiments but leakage of the sensitive attribute from representations is not sufficient on its own to demonstrate shortcut reliance.
\end{keytakeaway}

\subsection{Correlations}
\label{subsec:correlations}
Next, we provide the partial correlation values between TS and SA after conditioning on the BA model as discussed in Section~\ref{subsection:hypotheses}. Table~\ref{tab:rankcorr} shows these values for the three datasets/tasks, the two model architectures and the three partitioning methods. Note that the atlas-based partitioning is only available for the ADNI experiments. It can be seen that all correlations are moderate to high across datasets and partitions, indicating agreement between the TS and SA models beyond what can be explained by BA.

Pairwise and deviation-based correlation results are reported for the experiments presented in Sections~\ref{varyinganticorr} and~\ref{multiseedcomp}.

\begin{table}[htpb!]
  \centering
  \caption{Partial correlations
($\boldsymbol{\rho_{\mathcal{R}^{TS}\mathcal{R}^{SA}.\,\mathcal{R}^{BA}}}$)
  by dataset, model, and partitioning approach.}
  \label{tab:rankcorr}
  \setlength{\tabcolsep}{6pt}
  \renewcommand{\arraystretch}{1.2}
  \begin{tabular}{
    l l
    S
    S S
    S S
  }
    \toprule
    & & \multicolumn{1}{c}{\textbf{Atlas}}
      & \multicolumn{2}{c}{\textbf{Grid}}
      & \multicolumn{2}{c}{\textbf{Superpixel}} \\
    \cmidrule(lr){3-3} \cmidrule(lr){4-5} \cmidrule(lr){6-7}
    \textbf{Dataset} & \textbf{Model}
      & {—} & {8×8} & {16×16} & {64} & {256} \\
    \midrule
    CelebA & ResNet50 & N/A & 0.96 & 0.96 & 0.97 & 0.96 \\
              & InceptionV3 & N/A & 0.85 & 0.85 & 0.86 & 0.85 \\
    \midrule
    CheXpert & ResNet50 & N/A & 0.90 & 0.91 & 0.91 & 0.92 \\
              & InceptionV3 & N/A & 0.76 & 0.79 & 0.80 & 0.79 \\
    \midrule
    ADNI & ResNet50 & 0.89 & N/A & N/A & N/A & N/A \\
    \bottomrule
  \end{tabular}
\end{table}

\begin{keytakeaway}[]
\textbf{Conclusion}: OSCAR’s partial correlations are strongly positive across datasets and partitions, indicating consistent TS$\leftrightarrow$SA alignment beyond what is explained by BA.
\end{keytakeaway}

\subsection{Sensitivity Analysis}
\label{varyinganticorr}

The previous section showed that the partial correlation can identify when shortcut learning is taking place. In this section we perform a sensitivity analysis to show that it can also identify when shortcut learning is not present. To validate this, we systematically vary the number of `discordant' pairs $(\mathcal{Y},\mathcal{A})$ in the training set (i.e. minority pairs in the training data).
Specifically, in Table~\ref{tab:pergroup-train}, the number of discordant pairs was fixed to $25$ for each minority group for the 2D datasets. Here, we increased the discordant samples per group to $500 \to 1000 \to 1500 \to 2000$, so the total number of discordant samples across the two minority groups was $10\% \to 20\% \to 30\% \to 40\%$ of the training set. These increments effectively move the training datasets towards the BA training set where we would not expect shortcut learning to occur. 

\begin{figure}[htpb!]
\centering
\includegraphics[width=\linewidth]{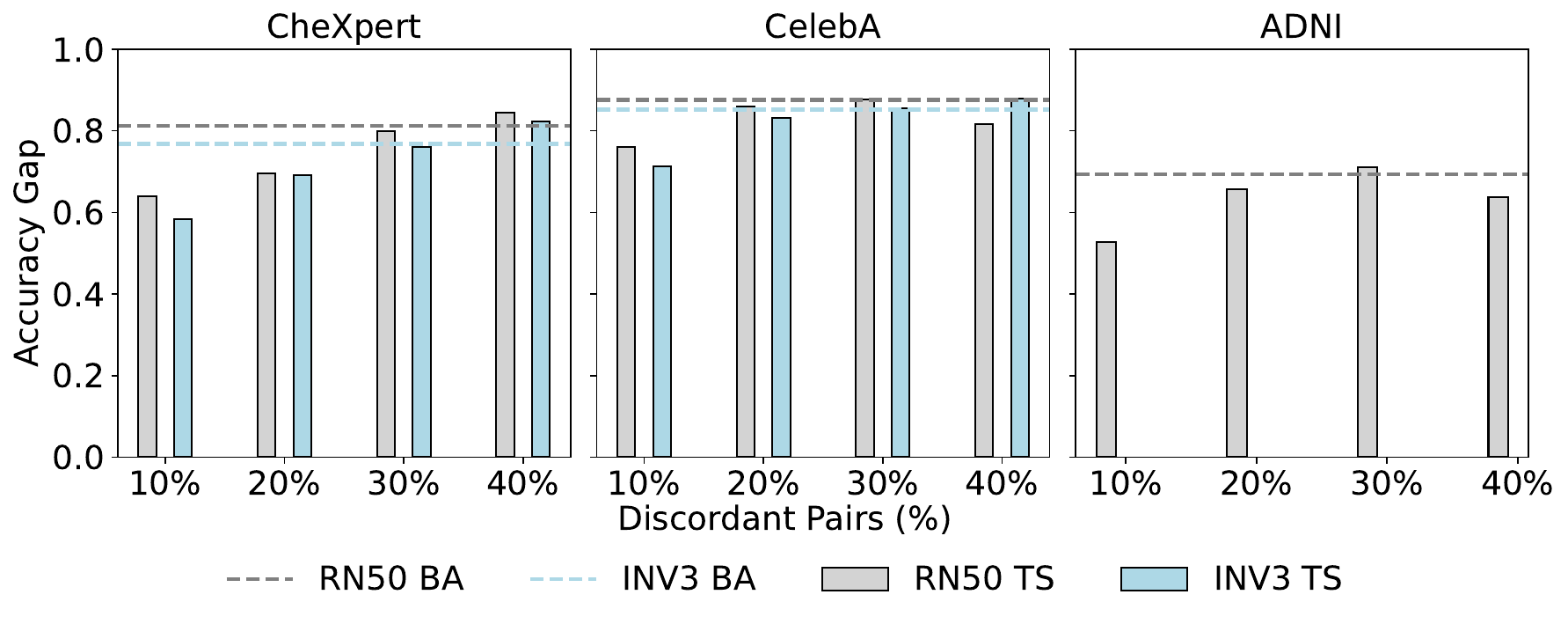}
  \caption{Worst group accuracy over varying number of discordant pairs for CelebA, CheXpert and ADNI datasets.}
  \label{fig:BAVAR}
\end{figure}

\begin{figure*}[htpb!]
\centering
\resizebox{\textwidth}{!}{
  \includegraphics[width=\linewidth]{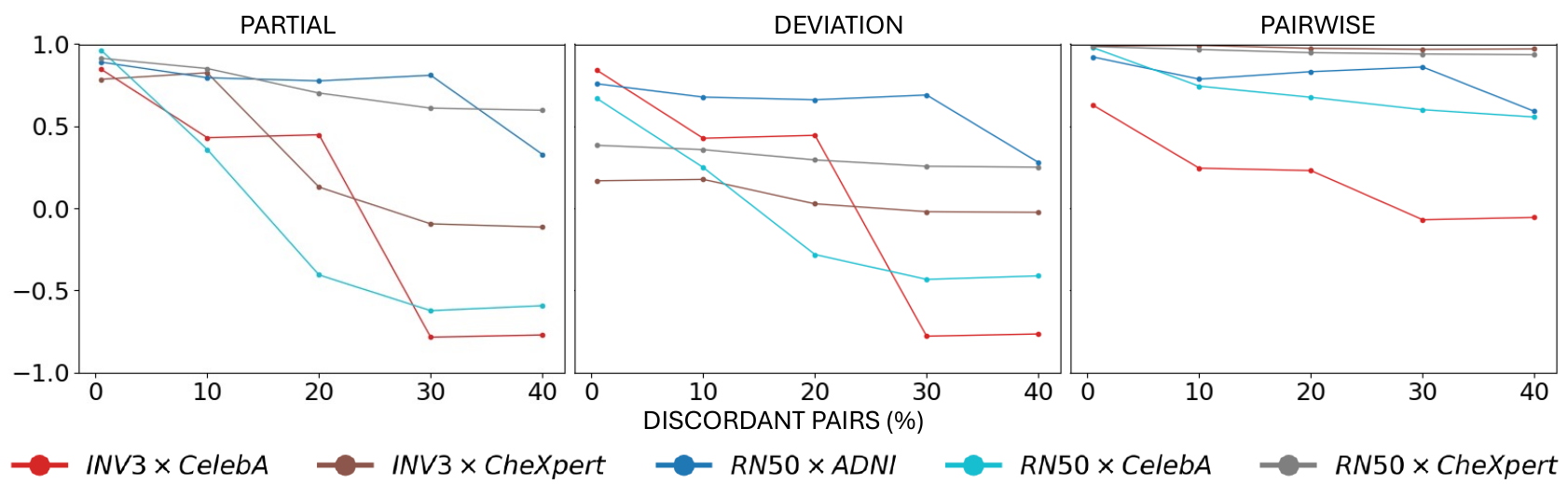}}
  \caption{Partial $\rho_{\mathcal{R}^{TS}\mathcal{R}^{SA}.\mathcal{R}^{BA}}$, deviation-based $\rho_{dev\ \mathcal{R}^{TS}\mathcal{R}^{SA}.\mathcal{R}^{BA}}$ and pairwise $\rho(\mathcal{R}^{TS}, \mathcal{R}^{SA})$ correlations over varying numbers of discordant pairs for CelebA, CheXpert and ADNI datasets using Grad-CAM attributions and $16\times 16$ grid partitioning. RN50 and INV3 stand for ResNet50 and InceptionV3 respectively.}
  \label{fig:shorttrends}
\end{figure*}

We observe in Figure~\ref{fig:BAVAR} that the worst group accuracy increases as the number of discordant training pairs increases, indicating that the shortcut learning effect is becoming less pronounced. Note that, in these plots, the disparity is measured between each bar and its corresponding baseline performance which is shown as a dotted line.

Figure~\ref{fig:shorttrends} shows the partial, deviation-based and pairwise correlation values for the same models. We can see that as the number of discordant pairs increases, the partial correlation values drop. Interestingly, for CelebA, as the number of discordant pairs increases, the partial and deviation-based correlations transition from strongly positive to negative, whereas for CheXpert they move towards zero. These negative correlations indicate active counter-alignment between TS and SA after controlling for BA. This behaviour is consistent with more spatially localised sensitive attribute features having contributed to shortcut-related behaviour. On CheXpert, the reduction in partial/deviation-based correlations (without becoming strongly negative) instead suggests a more diffuse or global sensitive attribute signal.
 Breaking the shortcut learning by balancing the training set reduces the alignment without inducing a spatially consistent counter-pattern. These findings are consistent with what we know of the sensitive attribute features in these two datasets: in CelebA, gender features are likely to be localised in specific areas of the face, whereas in CheXpert the sex-related features may be partially linked to the shadowing caused by breast tissue, which is spread across the image. However, these are human aligned explanations which the models may not follow.

 Overall, we find the partial correlation to be the most sensitive to shortcut learning. The pairwise correlations between TS and SA are less sensitive to the changes in the training dataset composition and are therefore less informative about the shortcut behaviour of the models.
The deviation-based correlation is more informative than pairwise correlation. However, for CheXpert, we do not observe significant changes in deviation-based correlation across varying numbers of discordant pairs, which could be due to the diffuse nature of the shortcut cues in CheXpert (as discussed above). These diffuse cues are detected better by the partial correlation.


\begin{keytakeaway}[]
\textbf{Conclusion}: The partial correlations decrease as discordant $(\mathcal{Y},\mathcal{A})$ pairs are added, capturing changes in shortcut strength as the training distribution is rebalanced.
\end{keytakeaway}

\subsection{RCS Analysis}

We also provide a qualitative illustration of how RCS maps change as the number of discordant pairs increases.
Figure~\ref{fig:RCS} shows RCS maps for the CelebA and CheXpert datasets for both ResNet50 and InceptionV3 model architectures. 
It can be seen how the CelebA models featuring shortcut learning (on the left of the figure) have RCS maps that strongly highlight gender-related features such as the lower part of the face. As the spurious correlations are reduced (moving to the right in Figure~\ref{fig:RCS}), the maps tend to highlight the hair more.
For CheXpert there is lower spatial concentration of features, reflecting the more diffuse nature of the shortcut features.

For the brain MRI dataset, RCS maps concentrate in partition cells that plausibly encode sex-related morphometric differences (see Figure~\ref{fig:adni_srel} and~\appref{appendix:subsection:ADNI}).

\begin{figure*}[htpb!]
\centering
\subfloat[CelebA, Grid ($16\times 16$)\label{fig:RCSanticorrresceleb}]{
  \includegraphics[width=0.75\linewidth]{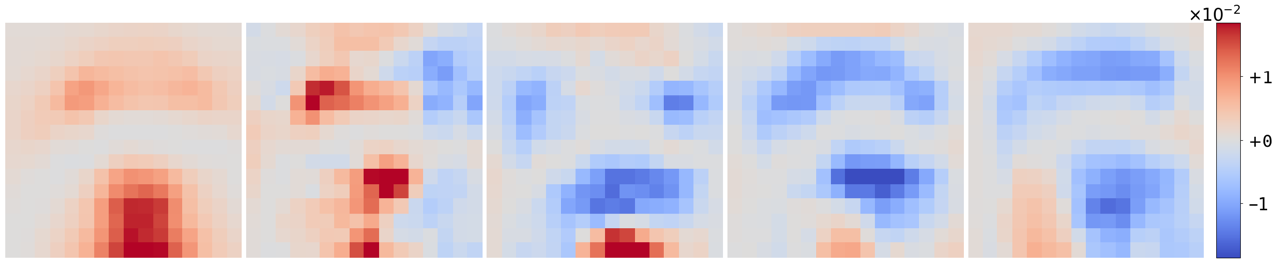}}\hfill
\subfloat[CelebA, Superpixel (k=64)\label{fig:RCSanticorrincceleb}]{
  \includegraphics[width=0.75\linewidth]{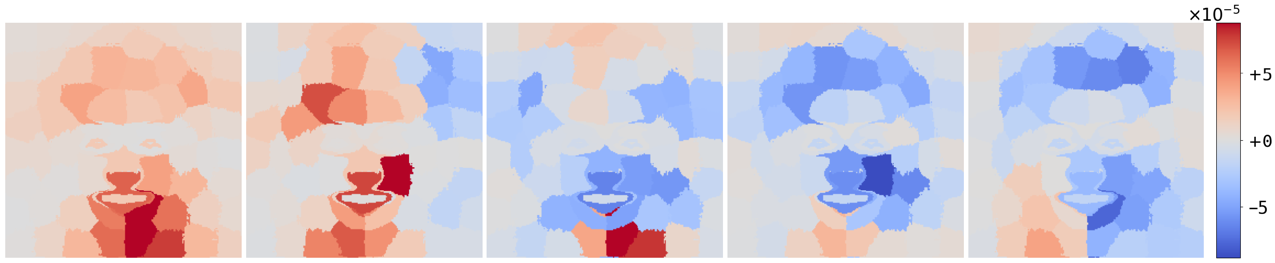}}\hfill
\subfloat[CheXpert, Grid ($16\times 16$)\label{fig:RCSanticorrreschex}]{
  \includegraphics[width=0.75\linewidth]{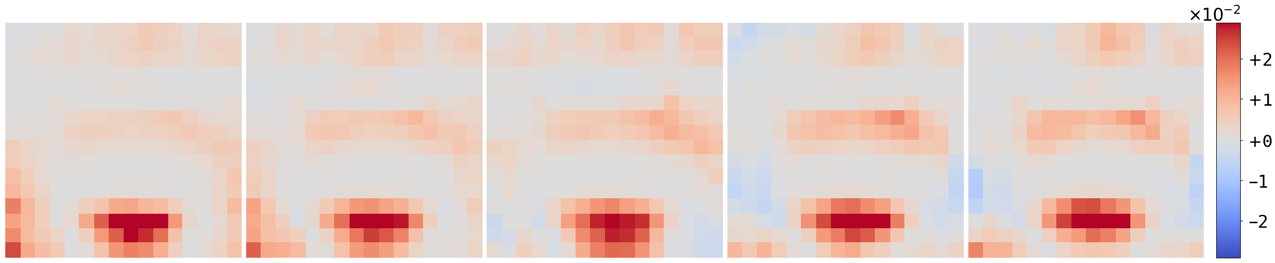}}\hfill
\subfloat[CheXpert, Superpixel (k=64)\label{fig:RCSanticorrincchex}]{
  \includegraphics[width=0.75\linewidth]{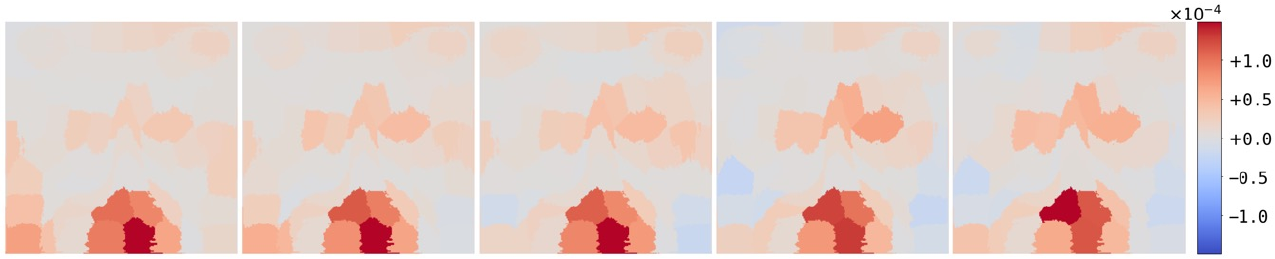}}
\caption{Region contribution scores (RCS) across increasing number of discordant pairs. Each row shows, from left-to-right, $25\  \text{samples} \to 10\% \to 20\% \to 30\% \to 40\%$ discordant pairs. All attributions were produced using Grad-CAM based on the grid-based $16\times16$ partitioning approach.}
\label{fig:RCS}
\end{figure*}

\begin{keytakeaway}[]
\textbf{Conclusion}: Region Contribution Score (RCS) maps localise which regions contribute most to the partial correlation, and hence to shortcut-related alignment.
\end{keytakeaway}

\subsection{Stability Analysis}
\label{multiseedcomp}
It is known that attribution maps can vary across different training seeds due to training time stochasticity introduced by factors such as data shuffling and model weight initialisation, i.e. the attributions can be unstable. Therefore, we analyse the stability of our method
by performing experiments in which the BA, TS and SA models are trained with different seed initialisations. Figure~\ref{fig:seeds} reports the pairwise, partial and deviation-based correlations for each model training run and shows that both ResNet50 and InceptionV3 remain highly stable across multiple seeds for both pairwise and partial correlations. The deviation-based correlation has slightly lower but still reasonable stability.

\begin{figure}[htpb!]
\centering
\resizebox{.5\textwidth}{!}{
  \includegraphics[width=\linewidth]{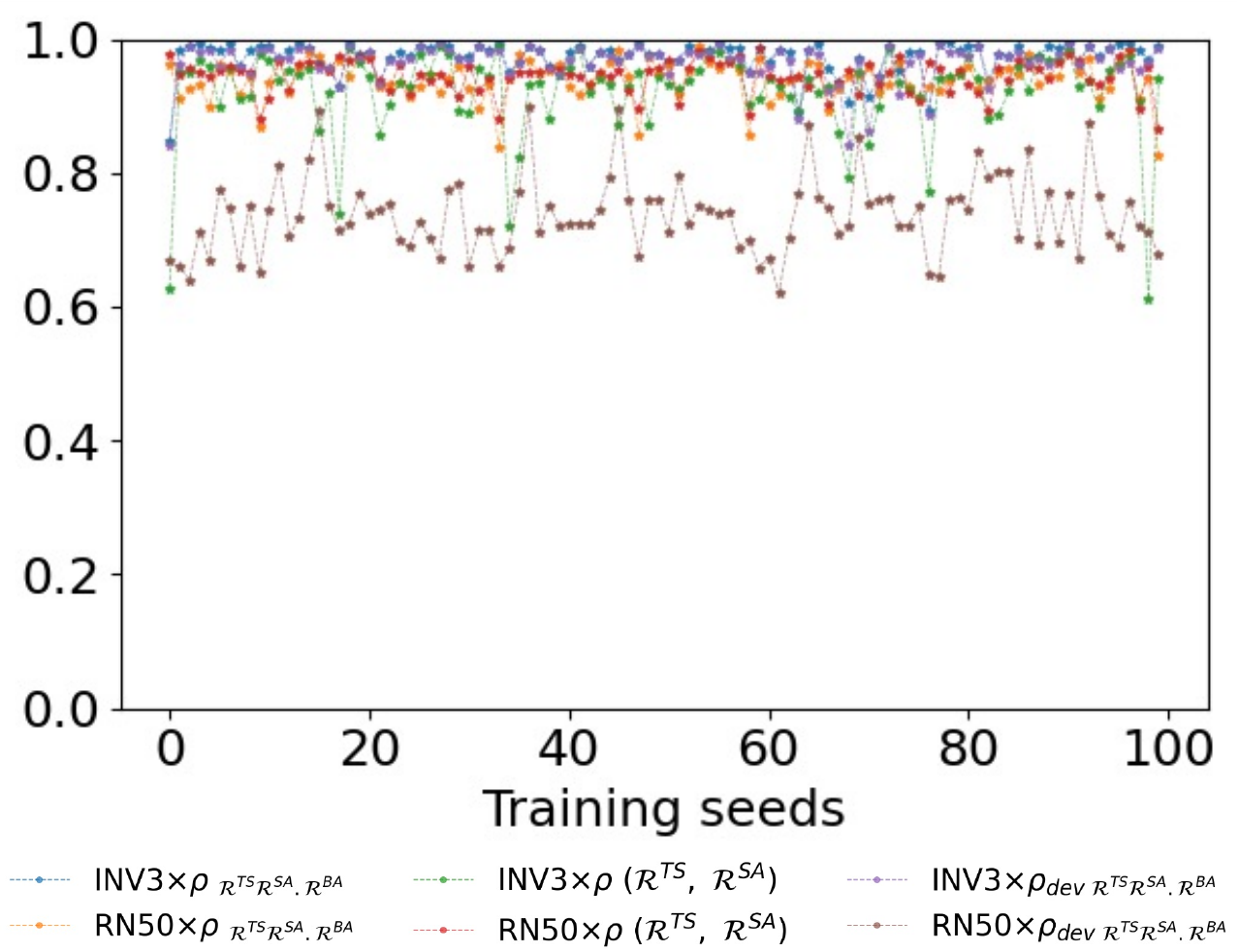}}
  \caption{Pairwise, partial and deviation-based (indicated as $\rho_{dev}$ correlation values for ResNet50 and InceptionV3 over multiple seeds based on the $16\times 16$ grid partitioning method using Grad-CAM attributions from the CelebA dataset.}
  \label{fig:seeds}
\end{figure}

\begin{keytakeaway}[]
\textbf{Conclusion}: The pairwise, partial, and deviation-based correlations are stable across random seeds, despite stochasticity in training and attribution maps.
\end{keytakeaway}

\subsection{Shortcut Mitigation}
\label{subsection:shortcutmitig}
Finally, we provide an illustration of the utility of the outputs of our method by showing how they can be used by a simple test-time attenuation technique to reduce the severity of shortcut learning.

Specifically, we used the RCS maps shown in Figure~\ref{fig:RCS} (left-most maps in every row) to weight regions based on their contribution to the partial correlations $TS \leftrightarrow SA\mid BA$ (shortcut-related) and $TS \leftrightarrow BA\mid SA$ (task-related). Firstly, we combine the corresponding RCS maps as:
\[
\mathrm{RCS}^{*} = \mathrm{RCS}(TS \leftrightarrow SA\mid BA) - \mathrm{RCS}(TS \leftrightarrow BA\mid SA).
\]
Note that positive values of $\mathrm{RCS}^{*}$ identify regions whose residual rank structure is more aligned with SA (shortcut-heavy), whereas negative values identify regions that are more aligned with BA (task-aligned).

During inference, we apply $\mathrm{RCS}^{*}$ in feature space just before the final global pooling layer. We interpolate $\mathrm{RCS}^{*}$ to match the spatial resolution of the penultimate feature maps and then construct a spatial weight mask $w$ that modulates the contribution of each location to the pooled feature vector. The mask is controlled by two hyperparameters: an ``up'' factor $\alpha$ and a ``down'' factor $\beta$. Regions with negative $\mathrm{RCS}^{*}$ (task-aligned) are scaled by $\alpha$, and regions with positive $\mathrm{RCS}^{*}$ (shortcut-aligned) are scaled by $\beta$. For small values of $\alpha$ and $\beta$, this behaves like a soft attenuation of shortcut-heavy regions and amplification of task-aligned regions; for more aggressive settings, the mask can effectively reverse the contribution of some regions, which is our intention. At test time, we obtain the penultimate feature maps, apply this RCS$^{*}$-derived mask, perform a weighted global pooling, and pass the pooled features to the final linear classifier (see Algorithm~\ref{alg:rcs_attenuation} in Appendix~\ref{appendix:rcs_attenuation} for further details).

The up-weight $\alpha$ and down-weight $\beta$ are hyperparameters
that control the scaling of positive and negative evidence. To find values for $\alpha$ and $\beta$, we
partition the original test split into four disjoint folds. For
each run, we treat three folds as a tuning set and one fold as a
held-out evaluation fold. We compute $\mathrm{RCS}^*$ and perform
a grid search over $(\alpha,\beta)$ using only the three tuning
folds, subject to the constraint that the worst-group accuracy on
the validation set does not decrease. We then fix $(\alpha,\beta)$
and evaluate the attenuated model on the held-out fold. Each fold is used once as the evaluation fold, and reported test metrics are averaged over the four evaluation folds.

Table~\ref{tab:attenuation} shows the results of this approach for CelebA using both ResNet50 and InceptionV3. It can be seen that
the inference time feature space weighting improves the overall and worst group performances. This is especially true for ResNet50 where there is about $17\%$ improvement in the balanced accuracy and $54\%$ improvement in worst group accuracy.

To verify that the weightings are truly meaningful, we also created random RCS-like maps in which we shuffled the original $\text{RCS}^*$ values by shifting each value to a new location which was at least $50\%$ of the map size away from its original location. We repeated this experiment over 10 different random RCS-like maps and report the average results along with the standard deviation. As can be seen in Table~\ref{tab:attenuation} (third row for each model), this results in a decrease (on average) in overall performance for both models and inconsistent changes in worst group accuracy.


\begin{table}[t]
\centering
\caption{Results for mitigation of shortcut learning using test-time attenuation based on RCS maps. Balanced accuracy (BAcc) and worst-group accuracy (WGAcc).}
\resizebox{.45\textwidth}{!}{
\begin{tabular}{l l l c c}
\toprule
Model & Dataset & Mask type & BAcc (\%) & WGAcc (\%) \\
\midrule
ResNet50 & CelebA & Baseline & 68.48 & 20.79 \\
         &        & RCS      & \textbf{85.1} & \textbf{74.8} \\
         &        & Random & $62.21 \pm 12.38$ & $21.14 \pm 19.78$ \\
\midrule
InceptionV3 & CelebA & Baseline & 70.19 & 17.19 \\
            &        & RCS      & \textbf{74.4} & \textbf{37.99} \\
            &        & Random & $66.44 \pm 7.53$ & $27.07 \pm 10.87$ \\
\midrule
ResNet50 & CheXpert & Baseline & 65.61 & 25.87 \\
         &        & RCS      & \textbf{66.25} & \textbf{30.25} \\
         &        & Random & $52.84 \pm 13.28$ & $21.2 \pm 15.02$ \\
\midrule
InceptionV3 & CheXpert & Baseline & 65.7 & 34.62 \\
         &        & RCS      & 65.1 & 33.36 \\
         &        & Random & $55.2 \pm 11.66$ & $26.1 \pm 11.45$ \\
\bottomrule
\end{tabular}}
\label{tab:attenuation}
\end{table}

\begin{keytakeaway}[]
\textbf{Conclusion}: RCS maps can be used to guide test-time attenuation in feature space, improving worst-group performance in settings with shortcut reliance.
\end{keytakeaway}

\section{Discussion}

Our results show that turning attribution maps into region-level rank profiles enables reproducible, dataset-level tests of shortcut use while preserving spatial meaning. Below, we first discuss methodological choices (block-level aggregation and partitioning), then analyse how shortcut correlations vary with increasing discordant training data, and finally consider implications for deployment, mitigation at test time, and the strengths and limitations of the framework.

\subsection{Why Block-Level Beats Pixel-Level}
Pixel-level saliency is sensitive to noise and model stochasticity, which complicates dataset-level inference. Aggregating to blocks and then ranking regions mitigates these issues by (i) smoothing idiosyncratic pixel noise, (ii) removing scale differences via ranks, and (iii) enabling hypothesis tests on compact, interpretable objects (the region ranks). Empirically, this stabilisation is reflected in the across-seed consistency of the correlations (Figure~\ref{fig:seeds}) despite training stochasticity. To check the robustness of block level metrics, we also compare “aggregate then rank” vs. “rank then aggregate” strategies. These results are presented in~\appref{appendix:aggrank}, in which it is shown that "aggregate then rank" is less informative compared to the chosen "rank then aggregate" strategy.

\subsection{Partitioning}
\label{sec:disc_partitions}
We evaluated fixed grid-based partitions and superpixel-based partitions. In addition, when an atlas is available, we also showed how it can be used to define more anatomically meaningful partitions (for the brain MRI dataset). Each choice has its benefits and limitations and the best choice will be dataset and application dependent, as discussed below.
\begin{itemize}[leftmargin=1.1em]
  \item \emph{Grid-based partitioning} ensures simplicity and reproducibility and can be adapted to the expected nature of the features (if known). Coarse grids favour diffuse cues, whereas fine grids might better capture localised cues, but this may come at the cost of increased noise if regions become too small.
  \item \emph{Superpixel-based partitioning} enables more human interpretable regions because it follows image edges. 
    However, it does rely upon the robustness of the superpixel algorithm.
  \item \emph{Atlas-based partitions} enable greater flexibility and highly interpretable features through the incorporation of prior anatomical knowledge. However, this fixes the partitions and is not data-driven which may introduce bias into the findings. Obviously, atlas-based partitioning also requires the presence of an atlas and adds the complexity of the alignment step to ensure that each image has anatomical correspondence with the atlas.
\end{itemize}

Practically, we recommend reporting multi-granularity results (i.e. both coarse and fine grid/superpixel) if possible, plus a domain-informed (i.e. atlas-based) partition where available. When atlas partitions vary widely in size, one should consider normalising region statistics by partition area or using order statistics (e.g., 90\textsuperscript{th} percentile) to avoid large partition cells dominating the rank profile. Our current approach of using the mean of the attribution values may also be a good choice.

\begin{figure}[htpb!]
\centering
  \includegraphics[width=\linewidth]{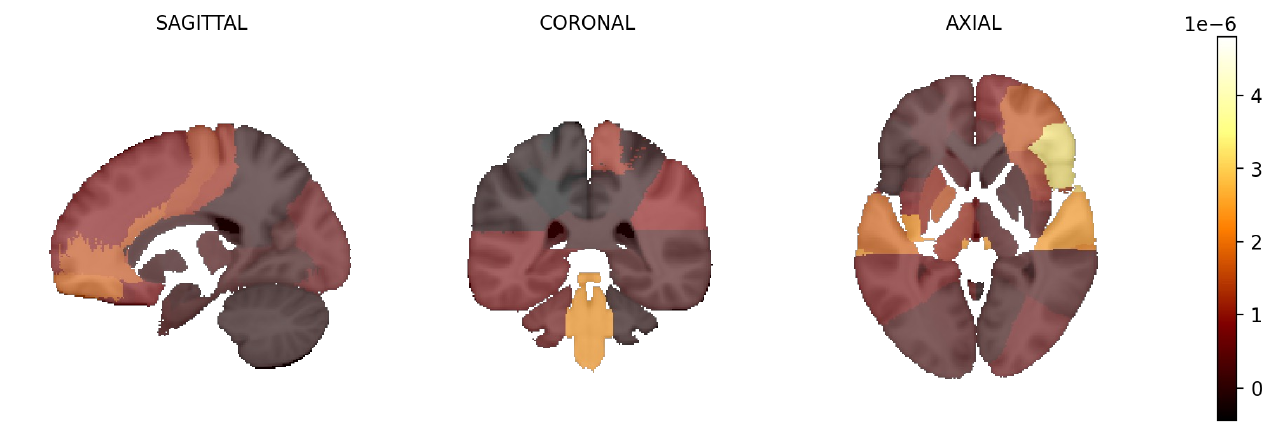}
  \caption{Region contribution score plot for the ResNet50 model on the ADNI dataset. Brighter colour indicates stronger contribution to the partial correlation $TS\ \leftrightarrow\ SA\ |\ BA$.}
  \label{fig:adni_srel}
\end{figure}

\subsection{What Increasing Discordant Pairs Reveals}

Our results have shown that, as the discordant pair subset grows, the partial correlation drops towards zero (CheXpert) or flips sign (CelebA). The sign flip carries useful information with regard to understanding the nature of shortcut learning. Specifically, it indicates \emph{counter-alignment}: once the shortcut is sufficiently disrupted, TS departs from SA in a spatially consistent way after controlling for BA. CheXpert’s monotonic decay without reversal suggests a more diffuse/global sensitive attribute signal (i.e. due to sex) in chest X-rays, which is harder to invert but still attenuated by counter-alignment. These patterns show that the partial correlation is sensitive not only to the \emph{presence} but also the \emph{structure} of the shortcut features (i.e. localised vs.\ diffuse).

Additionally, it can also be seen that different models can learn different types of shortcut feature given the same dataset. For instance, the CheXpert RCS maps for ResNet50 indicate the presence of diffuse cues whereas for InceptionV3, there is a suggestion of more localised cues (see the left corner of Figure~\ref{fig:RCSanticorrincchex}). Complementary RCS maps for $\,TS\leftrightarrow BA\mid SA\,$ (task-related regions) are shown in~\appref{appendix:potentiallytaskrelated}.

\subsection{Real World Implications}
Because our indices are computed on the balanced test set and conditioned on BA, they serve as a post-hoc test for deployment: if the partial correlation $\,\rho_{TS,SA.BA}$ is high (and RCS localises sensitive regions), a model likely relies on sensitive attribute-related cues. Iterative model training cycles can integrate our approach until the performance reaches a tolerable level.

Another choice that appears when developing models is how best to exploit the available data for training. Balancing the dataset based on one or more attributes often results in a reduction in training set size, and a consequent drop in performance. On the other hand, using all data may result in shortcut learning. Our framework can be used to check if the imbalance of a specific sensitive attribute is harmful and in cases where it is not, the dataset size need not be reduced. Additionally, in medical datasets, there might be multiple sensitive attributes in which case, decisions on balancing, etc. can be made based on the analyses on specific attributes.

\subsection{Are These Shortcuts Really Diffuse?}

The RCS maps in Figure~\ref{fig:RCS} show that the region-wise contributions to the partial correlation (for CheXpert) are very diffuse, which is indicative of shortcut learning features being spread across many regions. However, this cannot be said with certainty. One argument could be that the BA region-wise attributions will not be consistent across the dataset as the task-related visual evidence for pleural effusion is not fixed in  a single location. However, we can see from Figure~\ref{fig:shorttrends} that the pairwise correlation between TS and SA does not vary for CheXpert across different numbers of discordant examples (in contrast to CelebA). This potentially implies that the nature of shortcuts here is diffuse since the partial correlation drops with an increase in number of discordant pairs. Moreover, while these observations are related to ResNet50, with InceptionV3 the partial correlation drops to a small negative value when the discordant pairs are increased to 40\%. The regions positively contributing to the correlation do not flip to negatives and a significant number of diffuse blocks/cells are still present. More details about other correlations indicating the same are discussed in Appendix~\ref{appendix:potentiallytaskrelated}. There could also be potential subsets/slices where shortcuts interact differently. This can be analysed using subset level RCS maps.

\subsection{Test-Time Attenuation}
Our mitigation experiments highlight the possibility of performance improvements using test-time attenuation. There were significant performance gains for CelebA. For CheXpert, however, RCS-based weighting yields only modest improvements over the baseline. One plausible explanation is that the task-associated regions vary across CheXpert images, as pleural effusion is not localised in a fixed location, so upweighting is not targeting shortcut behaviour as specifically as in CelebA. Another aspect is that the layer targeted for test-time attenuation might not be the most effective one for mitigating bias; layer-wise analyses of shortcuts might provide more evidence here~\cite{tsoy2025measuring}.

We did not include mitigation results for ADNI (3D MRI). For this dataset, per-subject mitigation would require de-registering the population-level RCS maps into each subject’s native space, followed by recomputation or spatial transformation of attribution maps. This is a substantial engineering step orthogonal to the goal of OSCAR, which focuses on localising shortcut regions at the population level. Therefore, we leave native-space mitigation for ADNI as future work.

\subsection{Limitations and Strengths}
\label{sec:limitations}
Our framework is deliberately agnostic to the choice of interpretability method. This is a strength: it ensures broad applicability across architectures and modalities. In support of this, we also demonstrate OSCAR on a Vision Transformer (ViT) using the Attention-aware Layer-wise Relevance Propagation (AttnLRP) interpretability method, observing similar shortcut localisation behaviour (see \appref{appendix:transformermodels}).
Nevertheless, it does mean that our method (like any other) relies upon the quality and stability of the attributions - if an attribution method produces poor or uninformative maps, then any attribution-based shortcut detection method will suffer. In contrast, if a chosen attribution method faithfully captures model evidence, our rank-based aggregation translates that evidence into robust, testable hypotheses about shortcut use. Future extensions could include richer attribution ensembles, adaptive or anatomy-aware partitions beyond the atlas-based approach, and generalisations to multi-attribute and continuous settings. 

Furthermore, our method makes an assumption of approximate spatial alignment between image samples, which must be met for it to be successful. This is necessary to make the comparisons of attributions across regions meaningful. In our examples, cropped faces were approximately aligned in CelebA, chest X-rays also featured approximate alignment, and image registration was used to ensure alignment in the brain MRI dataset. Significant misalignment may negatively impact the correlations. The sensitivity of the correlations to spatial misalignment remains the focus of future work.

Another limitation is that our current method measures shortcut learning and identifies shortcut features at the dataset level. It might be that the use of shortcut features is limited to subsets, or `slices' of the dataset. Combining OSCAR with slice discovery methods~\cite{eyuboglu2203domino} would be an interesting area for future work.

\section{Conclusion}

We present OSCAR, a pixel-space auditing framework that tests and localises shortcut reliance by comparing aggregated rank profiles of attribution maps across three models: BA, TS, and SA. The resulting correlations of these profiles enable hypothesis testing of shortcut learning, while the RCS maps localise contributing regions. Across CelebA, CheXpert, and ADNI (brain MRI), we observed: (i) robust, spatially meaningful detection of shortcuts; (ii) sensitivity to localised vs.\ diffuse attribute signals; and (iii) stability across random seeds. 

We envision our framework as a lightweight audit step: if shortcut correlation is high, models should be retrained with counter-examples or balancing, then re-audited until the shortcut correlation falls below a pre-specified threshold. This complements standard accuracy/OOD tests. By pairing attribution with statistical inference in pixel space, our approach makes shortcut auditing more \emph{actionable}: it not only asks \emph{whether} a model shortcuts, but also \emph{where}.

\section*{Acknowledgments}
This research was supported by the UK Engineering and Physical Sciences Research Council (EPSRC) [Grant reference number EP/Y035216/1] Centre for Doctoral Training in Data-Driven Health (DRIVE-Health) at King's College London. 
Data used in preparation of this article were obtained from the Alzheimer’s
Disease Neuroimaging Initiative (ADNI) database (adni.loni.usc.edu). As such,
the investigators within the ADNI contributed to the design and implementation of ADNI and/or provided data but did not participate in analysis or
writing of this report. A complete listing of ADNI investigators can be found at:
\url{http://adni.loni.usc.edu/wp-content/uploads/how_to_apply/ADNI_Acknowledgement_List.pdf}.


\bibliographystyle{IEEEtran}
\bibliography{references}

 
\vspace{15pt}

\renewcommand{\thesubsectiondis}{\Alph{subsection}}

\appendix
\subsection{Additional Results}
\label{appendix:additionalresults}
\begin{figure}[htpb!]
\centering
\resizebox{.45\textwidth}{!}{
  \includegraphics[width=\linewidth]{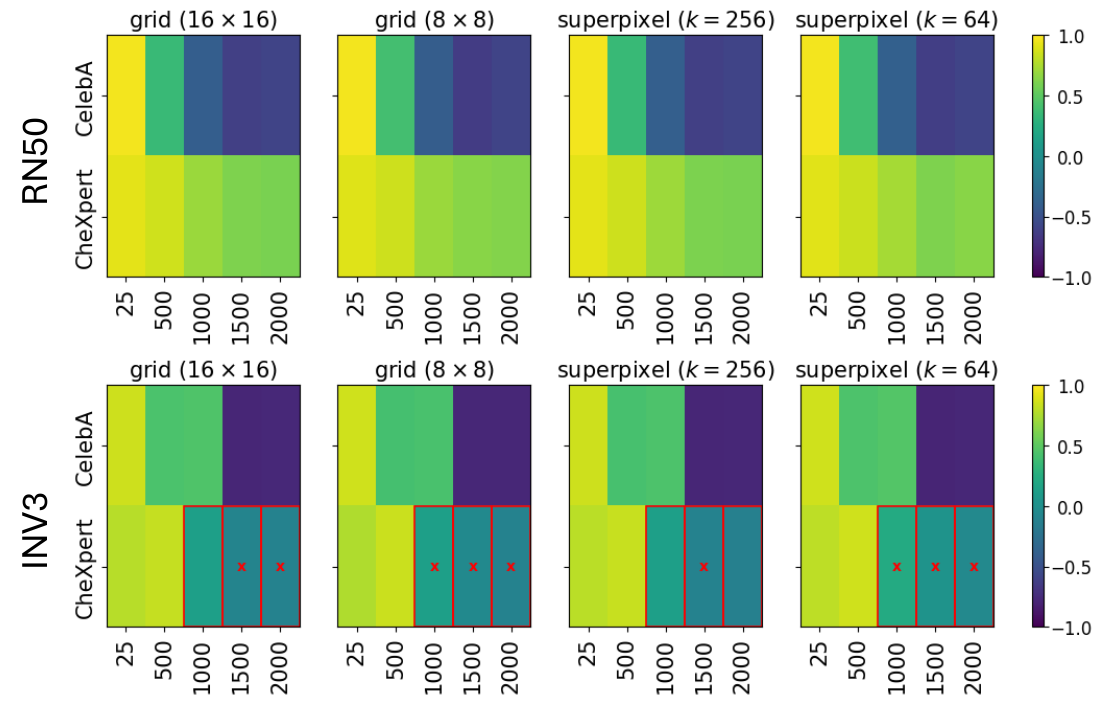}}
  \caption{Partial correlation heatmaps with 95\% image-level bootstrap confidence intervals. For each dataset, backbone, and spatial partition, the heatmaps show the point
estimates of the partial correlation 
$\rho_{\mathcal R^{TS}\mathcal R^{SA}.\mathcal R^{BA}}$
as the number of discordant $(\mathcal Y,\mathcal A)$ pairs in training increases.
Red ``$\times$'' markers denote entries whose 95\% bootstrap confidence interval
contains zero (equivalently, non-significant region-permutation tests at the 5\% level). RN50 and INV3 stand for ResNet50 and InceptionV3 respectively.}
  \label{fig:stsuperpixel}
\end{figure}

\subsubsection{Image-level bootstrap confidence intervals and region-permutation $p$-values}
\label{appendix:imagelevelcis}

For each dataset, model, and spatial partition, we report (i) a two-sided
permutation $p$-value computed from 10{,}000 random permutations of the region
indices, and (ii) a 95\% confidence interval obtained from 10{,}000
image-level bootstrap resamples.  
In each bootstrap replicate, test images are resampled with replacement,
aggregated rank profiles $(\mathcal R^{TS}, \mathcal R^{SA}, \mathcal R^{BA})$
are recomputed using the same median operator, and the partial correlation
$\rho_{\mathcal R^{TS}\mathcal R^{SA}.\,\mathcal R^{BA}}$ is re-evaluated.

The bootstrap intervals quantify sampling variability due to the finite test
set, while the permutation test evaluates whether the observed residual
alignment exceeds what would be expected under random spatial correspondence
across regions.  
Across grid, superpixel, and atlas partitions, both procedures show the same
trend with respect to the number of discordant pairs: correlation magnitudes
are largest when shortcut opportunities are strongest and decrease as shortcut
structure is reduced.  
Grid partitions typically yield smaller $p$-values because they induce lower
within-block spatial autocorrelation, resulting in a larger effective number of
independent regions.

Figure~\ref{fig:stsuperpixel} presents heatmaps of partial correlations, their
95\% confidence intervals, and significance indicators ($p > 0.05$) for all
datasets and partitions.

\subsubsection{Sensitive Attribute Classification}
\label{appendix:sensitiveattributeclass}

For completeness, we report the test-set accuracy of the sensitive attribute (SA)
models used in our correlation analyses. Each SA model is trained to predict the
sensitive attribute $\mathcal A$ from the images $\mathcal X$ using the same train/validation splits
as the BA model. Table~\ref{tab:sa-acc} reports the overall accuracy for
both ResNet50 and InceptionV3 across all datasets.

\begin{table}[!ht]
\centering
\caption{Test set accuracy (\%) of the Sensitive Attribute (SA) models.}
\label{tab:sa-acc}
\begin{tabular}{lccc}
\toprule
Dataset & Attribute & Model & Accuracy (\%) \\
\midrule
CelebA      & Gender & ResNet50    & 98.0 \\
CelebA      & Gender & InceptionV3 & 97.8 \\
CheXpert    & Sex    & ResNet50    & 92.4 \\
CheXpert    & Sex    & InceptionV3 & 88.6 \\
ADNI        & Sex    & ResNet50 (3D)    & 79.3 \\
\bottomrule
\end{tabular}
\end{table}

We additionally attempted an approach inspired by ShorT~\cite{brown2023detecting} linear attribute-subspace removal by projecting out the direction of sensitive attribute encoding identified by a linear probe on the penultimate layer features of the TS/BA models. Consistent with~\cite{brown2023detecting}, we find that removing this single direction has negligible effect on TS or BA predictions. This suggests that attribute reliance in our tasks is not mediated by a single linear representation direction, but instead reflects spatially structured or distributed signals. OSCAR’s region-level conditional correlations are able to detect these cases, whereas linear subspace removal is insufficient (see Section~\ref{varyinganticorr}).

\subsection{Interpretability Methods and Model Architectures}
\label{appendix:interpretabilitymethods}
While our approach is agnostic to the interpretability method used, we share the results on another method to show that the approach is not method-specific. Specifically, we compute signed relevance using LRP and then retain only the positive relevance (using ReLU) to focus on evidence supporting the predicted class. We highlight through the attribution maps in Figure~\ref{fig:vggattributions}, the alignment of attributions between TS and SA which focus on non-hair regions as compared to BA which focuses on the hair regions.

In Figure~\ref{fig:lrpseeds}, we observe a similar stability for LRP-based partial correlation values to that of Grad-CAM for both ResNet50 and InceptionV3 models.
This figure also shows results for other model architectures, namely MobileNetV3-Large and VGG16. The stability for the VGG model is better for LRP compared to Grad-CAM, whereas for MobileNetV3-Large, it is more stable for Grad-CAM compared to LRP. This shows that different models can work better with different kinds of interpretability method and therefore, a careful initial analysis can help in obtaining meaningful and stable correlations.

In Figure~\ref{fig:vggattributions}, we also share attribution maps from the CelebA dataset using the VGG model and LRP attributions to highlight the common attributions between the TS and SA models (focusing more on the facial region) and the BA model (focusing more on the hair).

\begin{figure}[tb!]
\centering
  \includegraphics[width=\linewidth]{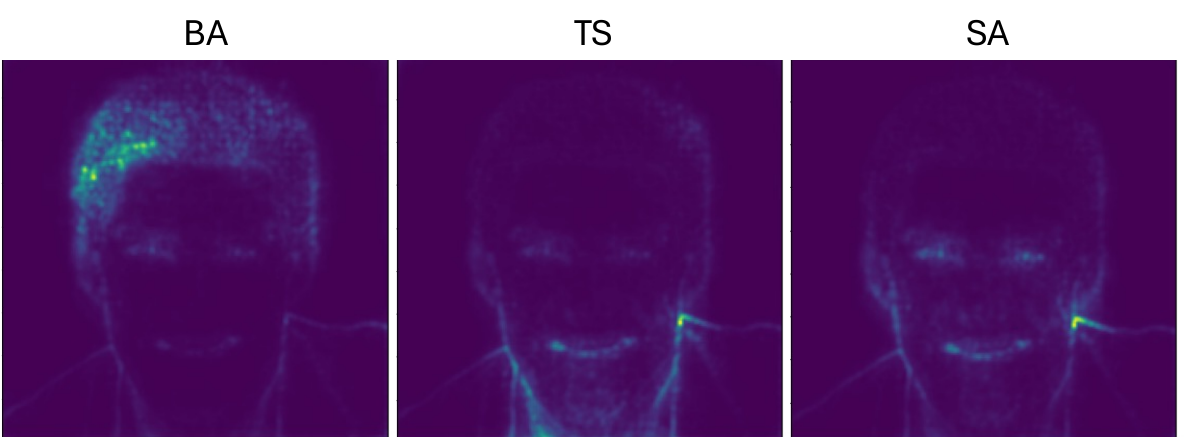}
  \caption{LRP attribution maps for the CelebA dataset and VGG16 model.}
  \label{fig:vggattributions}
\end{figure}

\begin{figure}[htpb!]
\centering
\resizebox{.45\textwidth}{!}{
  \includegraphics[width=\linewidth]{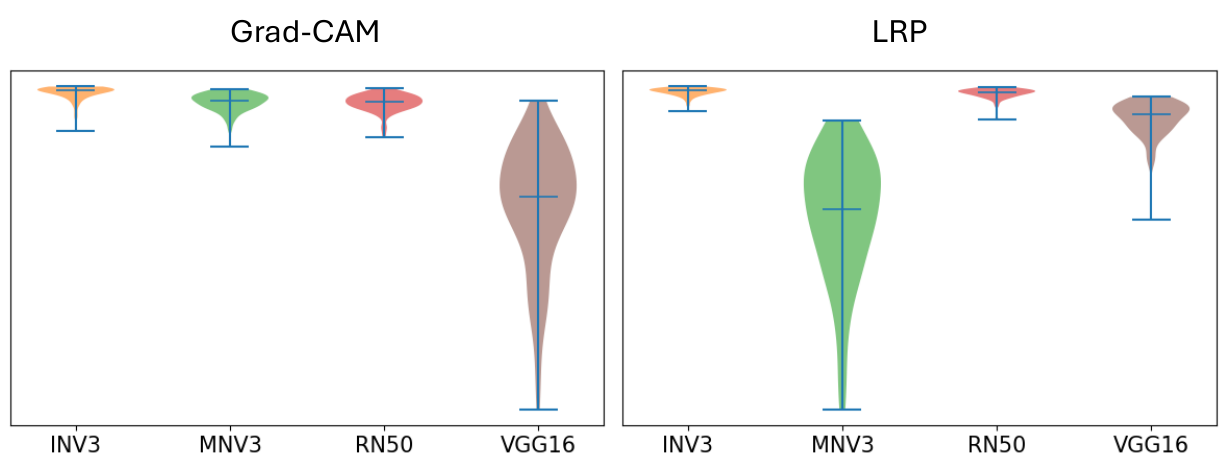}}
  \caption{Partial correlation $\rho_{\mathcal R^{TS}\ \mathcal R^{SA} . \ \mathcal R^{BA}}$ for ResNet50, InceptionV3, VGG and MobileNetV3-Large over multiple seeds based on the $16\times 16$ grid partitioning method using Grad-CAM (left) and LRP (right) attributions.}
  \label{fig:lrpseeds}
\end{figure}

\subsection{Theoretical Properties of OSCAR}
\label{appendix:theory}

This section presents basic properties of OSCAR statistics that follow directly from their construction. All notation follows Section~\ref{sectLnotationsetup} and Definition~\ref{def1:pcorr}.

\subsubsection{Rank-Based Nature and Monotone Invariance}
\label{appendix:subsec:rankmonotoneinvariance}

For model $M\in\{BA,TS,SA\}$ and image $i$, let
$\mathbf s_i^M = (s_{i,1}^M,\dots,s_{i,n}^M)$ denote the region-wise
attribution scores and
$\mathrm{rank}(\mathbf s_i^M)$ the vector of their ranks, breaking ties in any fixed
way.
The OSCAR profile for model $M$ is the component-wise median of these
rank vectors across images:

\begin{proposition}[Monotone Invariance]
\label{prop:monotone}
Let $g^M:\mathbb R\to\mathbb R$ be strictly increasing for each model
$M$, and define transformed scores
$\tilde s_{i,r}^M = g^M(s_{i,r}^M)$.
Let $\widetilde{\mathcal R}^{M}$ be the OSCAR profile constructed from
$\tilde s_{i,r}^M$.
Then, for every $M$, $\widetilde{\mathcal R}^{M} = \mathcal R^{M}$.

In particular, any statistic that depends only on
$(\mathcal R^{BA},\mathcal R^{TS},\mathcal R^{SA})$ (pairwise
correlations, partial correlations, deviation indices) is invariant to
strictly monotone transformations of the underlying attribution scores.
\end{proposition}

\begin{proof}
Fix $M$ and $i$ and consider regions $r_1,r_2$.
Since $g^M$ is strictly increasing,
\[
s_{i,r_1}^M < s_{i,r_2}^M
\;\Longleftrightarrow\;
g^M(s_{i,r_1}^M) < g^M(s_{i,r_2}^M)
\;\Longleftrightarrow\;
\tilde s_{i,r_1}^M < \tilde s_{i,r_2}^M.
\]
Thus the ordering of scores across regions is unchanged, so
$\mathrm{rank}(\mathbf s_i^M) =
 \mathrm{rank}(\tilde{\mathbf s}_i^M)$ for every image $i$.
Taking component-wise medians across images gives
$\widetilde{\mathcal R}^{M} = \mathcal R^{M}$.
\end{proof}

Proposition~\ref{prop:monotone} formalises that OSCAR works purely in
rank space: it is unaffected by the absolute scale of attribution
scores.

\subsubsection{Region Contribution Scores as a Decomposition of
Partial Correlation}

Recall from Definition~\ref{def1:pcorr} that the OSCAR partial
correlation between models $A$ and $B$ given $C$ is computed by
regressing $\mathcal R^A$ and $\mathcal R^B$ on $\mathcal R^C$, taking the
residuals $e_A,e_B$, and correlating these residuals.

In Section~III-H we define the Region Contribution Score (RCS) for
each region $r$ as
\[
z_{A,r} = \dfrac{e_{A,r} - \bar e_A}{\sigma_A},
\]
\[
z_{B,r} = \dfrac{e_{B,r} - \bar e_B}{\sigma_B},
\]
\[
\mathrm{RCS}(r) = z_{A,r}\,z_{B,r}
\]
where $\bar e_A,\bar e_B$ are the sample means and
$\sigma_A,\sigma_B$ the sample standard deviations of $e_A,e_B$
across regions.

\begin{proposition}[RCS as a Per-Region Decomposition of Partial Correlation]
\label{prop:rcs_decomposition}
Let $\rho_{AB\cdot C}$ denote the sample partial correlation defined in
Definition~\ref{def1:pcorr}, and let $\mathrm{RCS}(r)$ be as above.
Then
\[
\rho_{AB\cdot C}
=
\frac{1}{n-1}\sum_{r=1}^n \mathrm{RCS}(r).
\]
In particular, each $\mathrm{RCS}(r)$ is exactly the contribution of
region $r$ to the numerator of the partial correlation; summing over
regions recovers the full statistic.
\end{proposition}

\begin{proof}
By Definition~\ref{def1:pcorr},
\[
\rho_{AB\cdot C}
=
\dfrac{\sum_{r=1}^n (e_{A,r} - \bar e_A)(e_{B,r} - \bar e_B)}
     {\sqrt{\sum_{r=1}^n (e_{A,r} - \bar e_A)^2}
      \sqrt{\sum_{r=1}^n (e_{B,r} - \bar e_B)^2}}.
\]
The sample standard deviations are
\[
\sigma_A
=
\sqrt{\frac{1}{n-1}\sum_{r=1}^n (e_{A,r} - \bar e_A)^2},
\]
\[
\sigma_B
=
\sqrt{\frac{1}{n-1}\sum_{r=1}^n (e_{B,r} - \bar e_B)^2}.
\]
Hence,
\[
\sum_{r=1}^n z_{A,r} z_{B,r}
=
\sum_{r=1}^n
\frac{(e_{A,r} - \bar e_A)(e_{B,r} - \bar e_B)}{\sigma_A \sigma_B}
\]
\[
=
\frac{1}{\sigma_A \sigma_B}
\sum_{r=1}^n (e_{A,r} - \bar e_A)(e_{B,r} - \bar e_B).
\]
Substituting the expressions for $\sigma_A,\sigma_B$ shows that
\[
\rho_{AB\cdot C}
=
\frac{1}{n-1}\sum_{r=1}^n z_{A,r} z_{B,r}
=
\frac{1}{n-1}\sum_{r=1}^n \mathrm{RCS}(r),
\]
as claimed.
\end{proof}

In our implementation we additionally apply an $\ell_1$ normalisation
to the RCS values so that $\sum_r |\mathrm{RCS}(r)| = 1$, which
preserves the sign of each region's effect while putting magnitudes
on a common scale.
Proposition~\ref{prop:rcs_decomposition} shows that, before this
normalisation, the RCS map is an exact spatial decomposition of the
OSCAR partial correlation: each region's value reflects its signed
contribution to the overall shortcut alignment index.

\subsection{Transformer Models}
\label{appendix:transformermodels}
In this section, we provide results obtained using the Vision Transformer (ViT)~\cite{dosovitskiy2020image} model. We used Attention-aware Layer-wise Relevance Propagation (AttnLRP)~\cite{pmlr-v235-achtibat24a}, an attribution method tailored to Transformers.

We observe similar patterns for the partial correlation values as we do for other models. Figure~\ref{fig:vit_experiments}a shows the partial correlation variation across varying numbers of discordant samples, and Figure~\ref{fig:vit_experiments} shows the stability of correlations over 100 different seeds. Additionally, in Figure~\ref{fig:RCS_transformer} we provide RCS maps on the CheXpert dataset, where it can be seen that the chest area is highlighted as the task-related region. We observe similar RCS maps for CelebA as were seen for ResNet50.  These results further highlight the model-agnostic and attribution method agnostic nature of OSCAR.

\begin{figure*}[ht!]
\centering
\subfloat[Sensitivity\label{fig:vitdiscordant}]{
\resizebox{.45\textwidth}{!}{
  \includegraphics[width=0.49\linewidth]{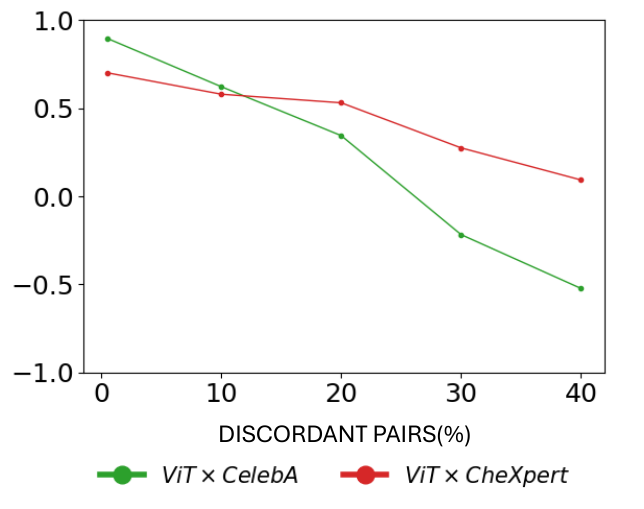}}}\hfill
\subfloat[Seed Stability \label{fig:vitstability}]{
  \includegraphics[width=0.49\linewidth]{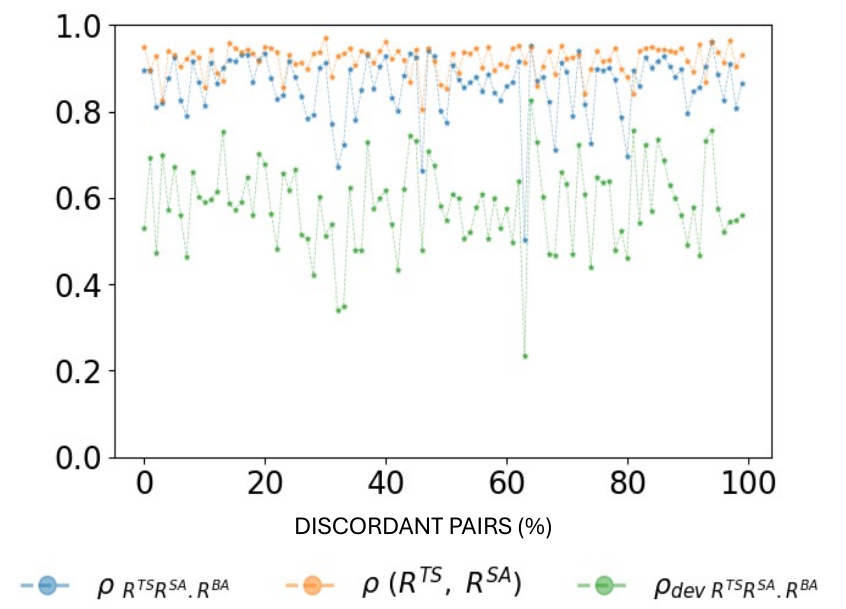}}
\caption{Sensitivity and stability analysis on ViT with AttentionLRP using grid-based ($16\times 16$) partitioning.}
\label{fig:vit_experiments}
\end{figure*}

\begin{figure}[!ht]
\centering
\resizebox{.45\textwidth}{!}{
  \includegraphics[width=\linewidth]{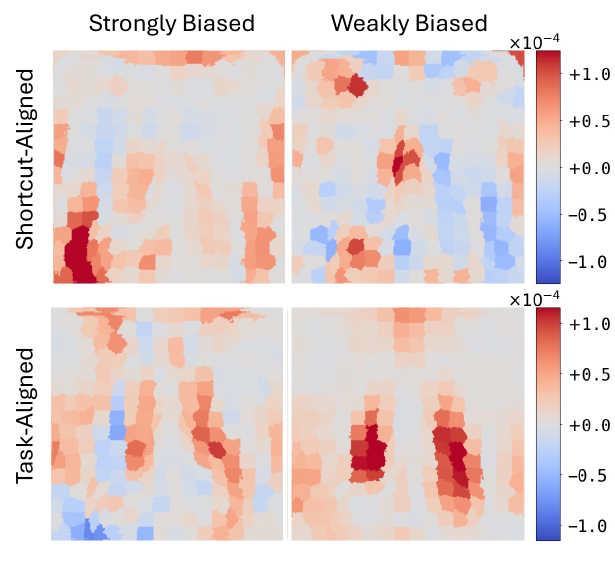}}
\caption{Region contribution score (RCS) maps across different numbers of discordant pairs using ViT on CheXpert dataset. Each row shows, from left-to-right, $25$  discordant pairs and $40\%$ discordant pairs. All attributions were produced using AttnLRP based on the superpixel ($k=256$) partitioning approach. The first row is computed using the correlation $\rho_{R^{TS}R^{SA}.R^{BA}}$ and the second using $\rho_{R^{TS}R^{BA}.R^{SA}}$ }
\label{fig:RCS_transformer}
\end{figure}

\subsection{Rank $\to$ Aggregate or Aggregate $\to$ Rank}
\label{appendix:aggrank}
In our work, we have first ranked the regions in individual images before aggregating the ranks across datasets. However, an alternate approach could be to first aggregate the mean attributions and then rank per dataset. We choose the former to be more robust to raw attribution deviations and share the variation of partial correlation with the discordant pairs in Figure~\ref{fig:srel_aggthenrank} to highlight the same. However, this is a design choice that can be explored depending on the specific scenario.

\begin{figure}[!ht]
\centering
\resizebox{.45\textwidth}{!}{
  \includegraphics[width=\linewidth]{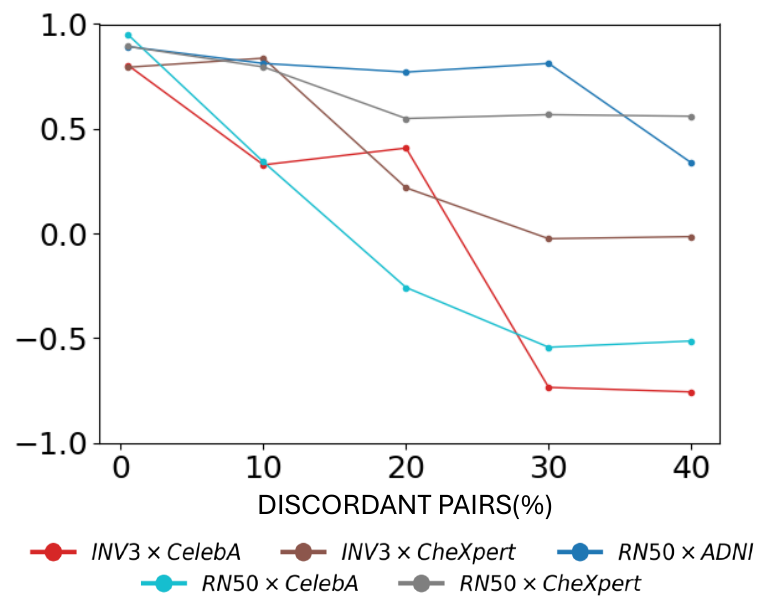}}
  \caption{Partial correlation trends for ResNet50 and InceptionV3 based on the $16\times 16$ grid partitioning method using Grad-CAM attributions. Here, we aggregate first and then rank the mean dataset attributions. RN50 and INV3 stand for ResNet50 and InceptionV3 respectively.}
  \label{fig:srel_aggthenrank}
\end{figure}



\subsection{Alternate Methods}
\label{appendix:salscoreagg}
We demonstrate results using three different alternatives that could be used in our framework (Figure~\ref{fig:srelsalscore}). Firstly, we used a saliency score of attribution regions prior to region level rank computation. To compute the Saliency score~\cite{stanley2022fairness}, we drop all the attributions below the $50^{th}$ percentile and then compute the ratio of number of the attribution pixels to the total number of pixels in a region.
Second, we use the mean of aggregated ranks instead of the standard median used in the main results. 

Finally, we discuss the Spearman's partial correlation instead of the Pearson's partial correlation used in main results.
We aggregate per-image region ranks into a dataset-level profile
$\mathcal R^M$ by taking the component-wise median across images.
The entries of $\mathcal R^M$ are therefore rank-based but also carry
information about how strongly a region tends to be preferred. For example, values
$39.5$ and $40.5$ would correspond to regions that are almost tied, whereas $10$ and
$80$ would indicate very different positions. Therefore, we investigate applying Pearson's
correlation directly to these aggregated ranks to capture both
directional agreement and the degree of mismatch in rank space.
Applying a second rank transform (i.e.\ Spearman's on $\mathcal R^M$)
would discard this local spacing information and treat all adjacent rank
differences as equally large, potentially over-penalising small
perturbations between nearly tied regions. Empirically, Spearman's yields
similar qualitative trends (Figure~\ref{fig:srelsalscore}), but Pearson's is slightly more
stable in the presence of near-ties and better reflects the strength of
alignment between region importance profiles.

\begin{figure*}[ht!]
\centering
    \resizebox{.85\textwidth}{!}{
  \includegraphics[width=\linewidth]{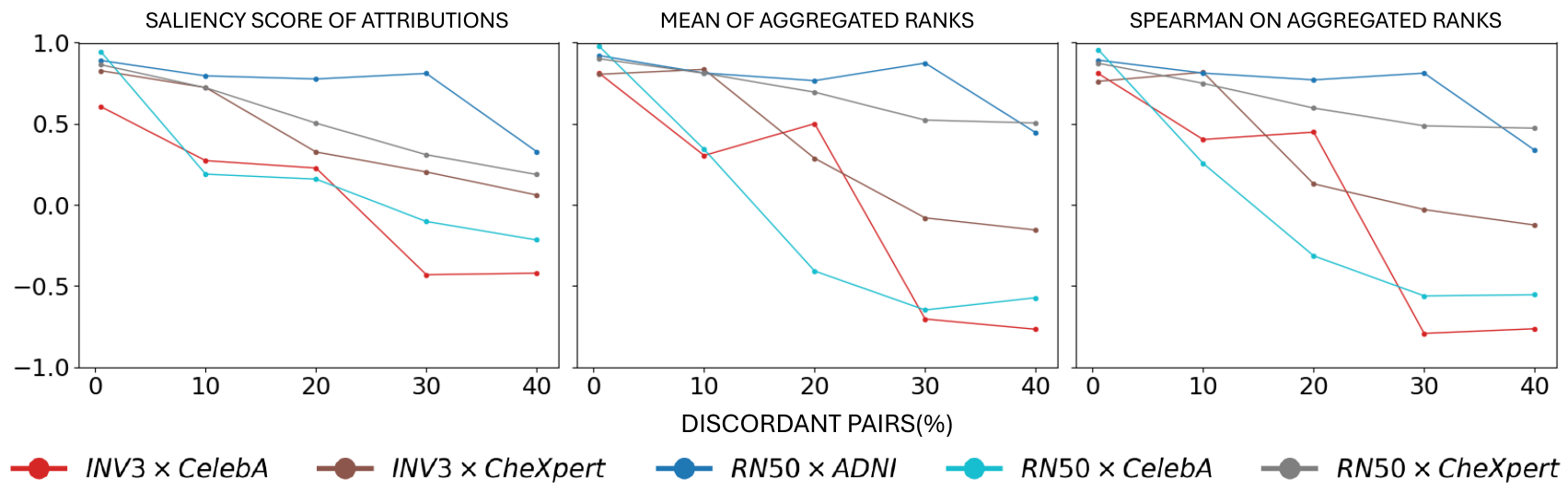}}
  \caption{Partial correlations for various alternative computations over varying numbers of discordant training pairs for CelebA and CheXpert datasets using Grad-CAM attributions.}
  \label{fig:srelsalscore}
\end{figure*}

\subsection{Potentially Task-Related and Non-Related Regions}
\label{appendix:potentiallytaskrelated}
In Figure~\ref{fig:RCS}, we showed the RCS maps computed using the partial correlation $TS \leftrightarrow SA | BA$, which indicated the regions that were potentially non-relevant and shortcut-aligned regions. For CelebA, we observed more localised cues as compared to CheXpert where the cues were more diffuse. To provide more insights, we also show in Figure~\ref{fig:Grid_useful} the RCS maps for the partial correlation $TS \leftrightarrow BA | SA$, which indicates regions that are more task-associated. It can be seen that the relevant regions (hair) for CelebA are not intersecting with the non-relevant regions, whereas these regions are more or less intersecting for CheXpert.

\begin{figure}[!ht]
\centering
\centering
    \resizebox{.45\textwidth}{!}{
  \includegraphics[width=\linewidth]{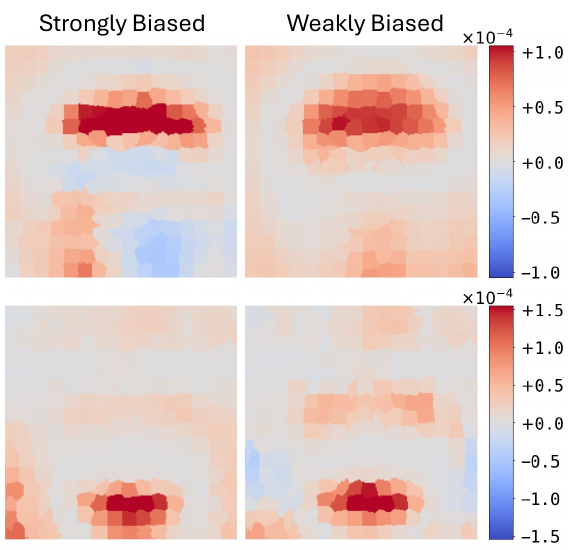}}
\caption{Region contribution scores (RCS) across different numbers of discordant pairs. Each row shows, from left-to-right, $25$ discordant pairs and $40\%$ discordant pairs. All attributions were produced using Grad-CAM (on ResNet50) based on the superpixel ($k=256$) based  partitioning approach. Here, the RCS maps are computed using the correlation $TS \leftrightarrow BA | SA$.}
\label{fig:Grid_useful}
\end{figure}

\subsection{Test-Time Attenuation with RCS-Based Feature Weighting}
\label{appendix:rcs_attenuation}

At test time, we keep the feature extractor fixed and replace the
standard global pooling with an RCS$^{*}$-based spatial weighting
applied to the penultimate feature maps before the final linear layer.
We write the network as $f = g \circ h$, where $h$ is the feature
extractor and $g$ is the final linear classifier.

\begin{algorithm}[!ht]
\caption{RCS$^{*}$-Based Weighted Pooling}
\label{alg:rcs_attenuation}
\begin{algorithmic}[1]
\REQUIRE
  input batch $x$,
  feature extractor $h$ and classifier $g$,
  RCS$^{*}$ mask at feature-map resolution,
  up-weight $\alpha > 0$, down-weight $\beta > 0$
\STATE Compute penultimate features
       \[
         F = h(x) \in \mathbb{R}^{B \times C \times H \times W}
       \]
\STATE Broadcast RCS$^{*}$ across batch and channels to obtain
       \[
         W \in \mathbb{R}^{B \times 1 \times H \times W}
       \]
\STATE For each image $b$, normalise by the maximum absolute value
       \[
         \tilde W_{b,1,u,v}
         =
         \frac{W_{b,1,u,v}}
              {\max_{u',v'} |W_{b,1,u',v'}| + \varepsilon }
       \]
       with $\varepsilon = 10^{-8}$
\STATE Define a scaling mask
       \[
         S_{b,1,u,v} =
         \begin{cases}
           \alpha, & \text{if } \tilde W_{b,1,u,v} < 0,\\[3pt]
           \beta,  & \text{otherwise,}
         \end{cases}
       \]
       and a final spatial weight
       \[
         A_{b,1,u,v} = 1 - \tilde W_{b,1,u,v}\, S_{b,1,u,v}
       \]
\STATE Compute weighted global pooling
       \[
         z_{b,c}
         =
         \frac{\sum_{u,v} A_{b,1,u,v}\, F_{b,c,u,v}}
              {\sum_{u,v} A_{b,1,u,v} + \varepsilon }
         \quad (b = 1,\ldots,B;c = 1,\ldots,C)
       \]
\STATE Obtain logits via the final linear layer:
       \[
         \hat y_b = g(z_b)
       \]
\STATE \textbf{return} logits $\hat y$
\end{algorithmic}
\end{algorithm}

Given a population-level RCS map $\mathrm{RCS}^{*}$ (Section~\ref{subsection:rcs}),
we first interpolate it to the spatial resolution of the penultimate
feature maps. Negative values in $\mathrm{RCS}^{*}$ correspond to
task-aligned regions (which we wish to upweight), while positive
values correspond to shortcut-aligned regions (which we wish to
downweight or potentially flip). This is controlled by two
hyperparameters: an up-weight factor $\alpha > 0$ and a
down-weight factor $\beta > 0$.

Algorithm~\ref{alg:rcs_attenuation} summarises the forward pass
with RCS$^{*}$-based feature weighting.

\subsection{Computational Details}
\label{appendix:computationaldetails}

We used the \texttt{timm} and \texttt{timm-3d} libraries for model instantiation and training. 
Attribution maps were computed via Grad-CAM (Captum), Layer-wise Relevance Propagation (LRP) (Zennit) and Attention-aware LRP (LXT). 
For Grad-CAM and LRP, we targeted the final convolutional or equivalent semantic layers of each backbone to ensure meaningful localisation:  

\begin{itemize}
    \item \textbf{ResNet50 (2D\&3D)}: \texttt{model.layer4[-1]}
    \item \textbf{MobileNetV3-Large}: \texttt{model.blocks[5][2]}
    \item \textbf{VGG16}: \texttt{model.features[29]}
    \item \textbf{InceptionV3}: \texttt{model.Mixed\_7c}
\end{itemize}

Training used the \texttt{Adam} optimiser with a learning rate of $5\times 10^{-5}$ and gradient clipping at a global norm of 0.5.
For data augmentation, no transforms were applied in the 2D settings; random rotations and translations were used in the 3D setting. For TS we use a biased validation split to encourage shortcut learning; BA uses a more balanced validation split matched to its training distribution. For the 3D model weight initialisation, the 2D ImageNet pretrained weights are converted to 3D based on \textit{timm-3d}\footnote{\url{https://github.com/ZFTurbo/timm_3d}}.

For stability analysis, we trained $100$ independent seeds per model; otherwise, a single seed was used. 

The grid search to find the optimal $\alpha$ and $\beta$ for the test-time attentuation (see Section~\ref{subsection:shortcutmitig}) is fixed for all the models and folds. Additionally, the parameters are searched to find the optimal worst-group accuracy while keeping the balanced accuracy within $0.5\%$ of the optimal and/or initial balanced accuracy.

\subsection{Additional Analysis on ADNI}
\label{appendix:subsection:ADNI}

Figure~\ref{fig:adni_attrmaps} shows the aggregated atlas-based attribution ranks over the ADNI test set for the SA and TS models. The SA map displays a distributed pattern that includes medial temporal, subcortical and brainstem regions, which is compatible with MRI work reporting sex differences in regional brain morphology and asymmetry, as well as sex-specific trajectories of brain aging in middle-aged and older adults \cite{coffey1998sex,good2001cerebral,ruigrok2014meta,armstrong2019sex}. However, given the moderate SA performance (79.3\% accuracy), we interpret these maps as indicating regions carrying sex-predictive signal for the model rather than as a precise localisation of established sex effects. The aggregated ranks for the TS model are highly similar to the SA model. Potential shortcut-related alignment between Alzheimer’s and sex attributions, quantified via partial correlations and Region Contribution Score (RCS) maps (see Section~\ref{sec:disc_partitions}), is illustrated in Figure~\ref{fig:adni_srel}.

\begin{figure}[ht!]
\centering
\subfloat[SA\label{fig:adnibaseline}]{
\resizebox{.5\textwidth}{!}{
  \includegraphics[width=0.49\linewidth]{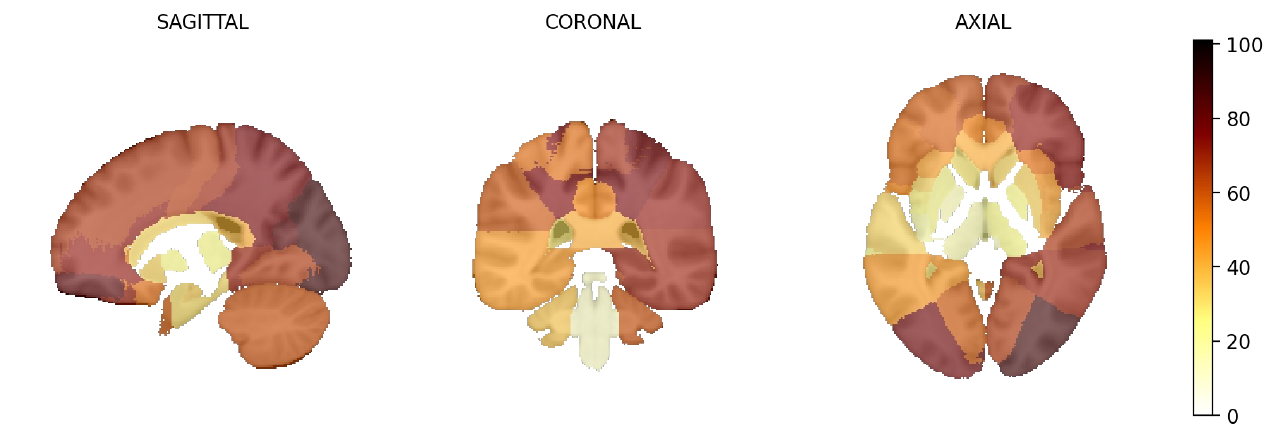}}}\hfill
\subfloat[TS\label{fig:adnisex}]{
\resizebox{.5\textwidth}{!}{
  \includegraphics[width=0.49\linewidth]{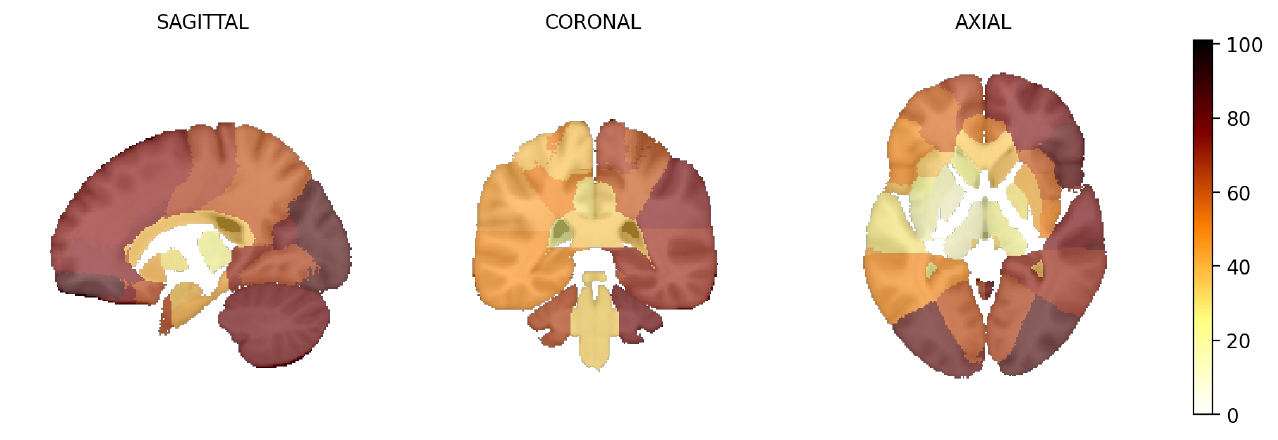}}}
\caption{Aggregated attribution ranks over the ADNI test set for the atlas based partition. Brighter regions indicate higher rank which implies higher region importance for the task. Scaling is done separately for SA, TS (per row).}
\label{fig:adni_attrmaps}
\end{figure}

\end{document}